\documentclass[11pt,letterpaper]{article}

\usepackage[utf8]{inputenc}
\usepackage{geometry}
\usepackage{amsmath,amssymb,amsfonts,amsthm}
\usepackage{graphicx}
\usepackage{anyfontsize}
\usepackage{booktabs}
\usepackage{tabularx}
\usepackage{natbib}
\usepackage{hyperref}
\usepackage{cleveref}
\usepackage{color}
\usepackage{titlesec}
\usepackage[font=footnotesize,labelfont=bf]{caption}

\title{\textbf{HamiFormer: Dual-Expert Diffusion Fields\\ with Affine Symplectic Maps}}
\author{%
  \textbf{Haoxiang Huang}$^{1}${\normalfont,}\enspace
  \textbf{Xiang Liu}$^{2}${\normalfont,}\enspace
  \textbf{Shuwei Wang}$^{3}${\normalfont,}\enspace
  \textbf{Jingheng Ma}$^{2}${\normalfont,}\enspace
  \textbf{Sen Cui}$^{2,\dagger,\ddagger}$\\[0.62em]
  {\footnotesize\sffamily\mdseries\color{black}
    $^{1}$ University of Science and Technology of China{\normalfont,}\enspace
    $^{2}$ Tsinghua University{\normalfont,}\enspace
    $^{3}$ Chinese Academy of Sciences}}
\date{}

\usepackage{xcolor}
\usepackage{environ}
\usepackage{tcolorbox}
\usepackage{tikz}
\usetikzlibrary{positioning, arrows.meta, calc, fit, shapes.geometric, backgrounds}
\usepackage{titlesec}
\usepackage{colortbl}
\usepackage{array}
\usepackage{placeins}
\usepackage{needspace}
\usepackage{enumitem}
\usepackage{mathtools,bm,multirow}
\usepackage{algorithm,algpseudocode}
\usepackage{capt-of}

\definecolor{paperblue}{RGB}{41,105,176}
\definecolor{accent}{RGB}{230,90,50}
\definecolor{soft}{RGB}{232,240,250}
\definecolor{tableheadgray}{RGB}{246,246,246}
\definecolor{metricgain}{RGB}{0,145,90}
\definecolor{metricloss}{RGB}{205,45,55}
\colorlet{linkblue}{paperblue}

\hypersetup{
  colorlinks=true,
  linkcolor=linkblue,
  citecolor=linkblue,
  urlcolor=linkblue,
  filecolor=linkblue,
  linktoc=all,
  bookmarksopen=true,
  bookmarksnumbered=true,
  pdfborder={0 0 0}
}

\graphicspath{{figures/}}

\titleformat{\section}
  {\color{paperblue}\normalfont\Large\bfseries\sffamily}
  {\thesection}{1em}{}
\titleformat{\subsection}
  {\color{paperblue}\normalfont\large\bfseries\sffamily}
  {\thesubsection}{1em}{}
\titleformat{\subsubsection}
  {\color{paperblue}\normalfont\normalsize\bfseries\sffamily}
  {\thesubsubsection}{1em}{}
\titleformat{\paragraph}[runin]
  {\normalfont\normalsize\bfseries\sffamily}{}{0pt}{}
\titlespacing*{\section}{0pt}{2.8ex plus 0.6ex minus 0.2ex}{1.0ex plus 0.2ex}
\titlespacing*{\subsection}{0pt}{2.3ex plus 0.5ex minus 0.2ex}{0.8ex plus 0.1ex}
\titlespacing*{\subsubsection}{0pt}{1.9ex plus 0.4ex minus 0.2ex}{0.7ex plus 0.1ex}
\titlespacing{\paragraph}{0pt}{1.2ex plus 0.3ex minus 0.1ex}{0.7em}

\renewenvironment{abstract}{%
  \begin{tcolorbox}[
    colframe=paperblue,
    colback=white,
    arc=8pt,
    boxrule=0.5pt,
    left=15pt,
    right=15pt,
    top=12pt,
    bottom=12pt,
    before upper={\small\setlength{\parskip}{4pt}\setlength{\parindent}{0pt}}
  ]
  \begin{center}
    {\color{paperblue}\large\bfseries\sffamily Abstract\par}
    \vskip 0.5em
  \end{center}
}{%
  \end{tcolorbox}
}

\renewcommand{\textbf}[1]{{\sffamily\bfseries #1}}

\makeatletter
\renewcommand{\maketitle}{%
  \newpage
  \thispagestyle{plain}%
  \null
  \vskip -0.8em
  \setbox\@tempboxa=\vbox{%
    \hsize=\textwidth
    \parskip=0pt
    \centering
    {\Large\sffamily\bfseries \@title\par}%
  }%
  \dimen@=\ht\@tempboxa
  \advance\dimen@ by \dp\@tempboxa
  \advance\dimen@ by 3em
  {\parskip=0pt
    \noindent\vbox{%
      \hsize=\textwidth
      {\color{paperblue}\hrule height 2pt}%
      \vbox to \dimen@{\vfil\box\@tempboxa\vfil}%
      {\color{paperblue}\hrule height 0.5pt}%
    }\par
  }%
  \vskip 1.10em
  {\centering
    {\large\sffamily
      \lineskip .5em%
      \begin{tabular}[t]{c}%
        \@author
      \end{tabular}\par}%
  }%
  \vskip 1.20em
}
\makeatother

\newcommand{\R}{\mathbb{R}}

\newtheoremstyle{paperbold}
  {6pt}{6pt}{\normalfont}{0pt}{\sffamily\bfseries}{.}{0.5em}{}
\theoremstyle{paperbold}

\newtheorem{proposition}{Proposition}

\newtheorem{theorem}{Theorem}

\newcommand{\E}{\mathbb{E}}
\newcommand{\Id}{\mathrm{I}}
\newcommand{\diag}{\operatorname{diag}}

\newcommand{\logit}{\operatorname{logit}}
\newcommand{\clip}{\operatorname{clip}}
\newcommand{\MSE}{\operatorname{MSE}}

\newcommand{\HR}{H{+}r}

\definecolor{resultfirst}{RGB}{226,249,218}
\definecolor{resultsecond}{RGB}{255,250,210}
\definecolor{resultthird}{RGB}{255,226,226}
\newcommand{\resultfirst}[1]{\cellcolor{resultfirst}\textbf{#1}}
\newcommand{\resultsecond}[1]{\cellcolor{resultsecond}#1}
\newcommand{\resultthird}[1]{\cellcolor{resultthird}#1}

\newcommand{\norm}[1]{\left\lVert#1\right\rVert}
\newcommand{\valpha}{\bm{\alpha}}
\newcommand{\vdelta}{\bm{\delta}}
\newcommand{\vepsilon}{\bm{\epsilon}}
\newcommand{\vmu}{\bm{\mu}}
\newcommand{\vphi}{\bm{\phi}}
\newcommand{\vpsi}{\bm{\psi}}
\newtheorem{corollary}[theorem]{Corollary}

\let\paperOriginalNormalsize\normalsize
\renewcommand{\normalsize}{%
  \paperOriginalNormalsize
  \fontsize{10.95}{13.1}\selectfont
}
\begin{document}
\normalsize

\maketitle
\begin{abstract}
Predicting smooth dynamics and collisions requires modeling continuous evolution and abrupt state changes. We introduce HamiFormer, a dual-expert diffusion field combining whole-window denoising with residual-corrected Hamiltonian propagation. Their mixed-state feedback attenuates the direct contribution of inherited autoregressive error: each mixed state guides subsequent propagation within the jointly refined window. Parallel Local Affine Scan (PLAS) amortizes iterative refinement across rectified-flow steps and evaluates derivatives in parallel across physical time. PLAS's affine symplectic maps achieve lower solver error and runtime than sequential explicit Euler in our evaluation. A Regime Model Tree specializes residuals and routing to balance typical-state accuracy against large tail errors. Our analysis gives conditions for physically consistent refinement and warm-start tracking, and finite-window error bounds under diffusion feedback. In 192-step evaluations, HamiFormer reduces normalized phase-space MSE by 26.3\% against PhysiFormer on HamiBalls-1 and 21.4\% against DiT on HamiBalls-2, with comparable model capacities. Across disjoint intervals, Ours has the lowest late-horizon position and momentum errors among baselines on both datasets.
\vspace{4pt}

{\footnotesize
{\textbf{Date:}} September 26, 2026\\
{\textbf{Project page:}} \href{https://hamiformer.github.io/}{\texttt{https://hamiformer.github.io/}}\hfill $^{\dagger}$ Corresponding author\\
{\textbf{Code, datasets and weights:}} \href{https://github.com/starx237/HamiFormer}{\texttt{https://github.com/starx237/HamiFormer}}\hfill $^{\ddagger}$ Project leader
}
\end{abstract}

\enlargethispage{5pt}
\begin{figure}[!ht]
\centering
\vspace{2.75pt}
\includegraphics[width=0.95\textwidth,trim=0 6bp 0 5.5bp,clip]{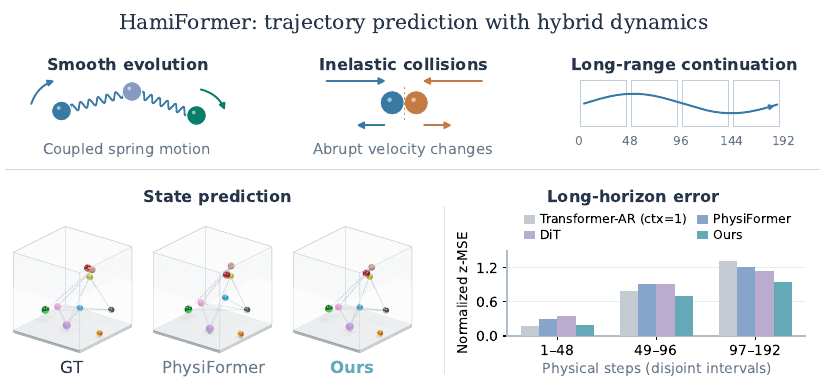}
\vspace{2.75pt}
\caption{HamiFormer overview: schematic hybrid dynamics (top), with ground-truth and predicted HamiBalls-2 states at the same step and normalized phase-space MSE over three disjoint intervals spanning 192 steps (bottom).}
\label{fig:introduction-teaser}
\end{figure}
\FloatBarrier
\section{Introduction}
\begingroup
\linespread{1.045}\selectfont
\setlength{\parskip}{10pt minus 1.5pt}
Learned physical simulators must reconcile local dynamics with long-range prediction. A one-step predictor can be accurate on observed states yet encounter progressively different inputs when recursively applied to its own outputs, allowing prediction errors to accumulate over time. This distribution shift makes long-horizon reliability difficult to infer from one-step prediction accuracy. Whole-window generative prediction offers a complementary approach by revising all future states together and updating distant physical times. PhysiFormer adopts this perspective for world-space mechanics \citep{chen2026physiformer}. Combining local state propagation with whole-window denoising requires coordinating their predictions across physical time within each jointly refined future window.

We parameterize local propagation through a learned scalar Hamiltonian, whose derivatives define smooth mechanical evolution \citep{greydanus2019hnn}. However, collisions, dissipation, and model mismatch also require a flexible prediction pathway beyond the smooth Hamiltonian model. The challenge is therefore to integrate local dynamics into a whole-window diffusion model, with its corrections fed back into subsequent propagation.

We propose \textbf{HamiFormer}, whose clean trajectory estimator contains two experts. The diffusion expert predicts the entire future window. The Hamiltonian expert combines a scalar Hamiltonian model with a learned residual, denoted $H+r$. At each physical edge, a router mixes their candidates, and this mixed state becomes the predecessor of the next Hamiltonian step. This feedback provides a \emph{soft reset}: mixing with the whole-window diffusion candidate attenuates the direct contribution of inherited autoregressive error before subsequent Hamiltonian propagation. At each denoising step, this mixed-state propagation re-estimates the entire future trajectory.

Repeated nonlinear propagation inside a diffusion sampler introduces a second challenge. General nonseparable Hamiltonians require implicit symplectic-Euler solves throughout the physical trajectory. Building on Newton-based sequence linearization \citep{lim2024deer}, we develop \textbf{Parallel Local Affine Scan (PLAS)} for closed-loop dual-expert denoising. PLAS accommodates general Hamiltonians without imposing a prescribed separable form. It linearizes implicit Hamiltonian relations at cached trajectory anchors, constructing affine maps from gradients and Hessians evaluated in parallel across physical time. These maps act on newly mixed predecessors; nonlinear residuals and routing respond to the evolving mixed history. Anchors persist across rectified-flow (RF) steps, carrying refinement progress between changing denoising conditions. Under suitable local conditions, this selective refinement tracks the evolving mixed solution, amortizing iterative work across RF evaluations. Each PLAS-SymEuler affine map preserves symplectic structure exactly in the Hamiltonian propagation step, before residual correction and diffusion mixing. In our solver evaluation, the resulting propagation achieves lower numerical error and runtime than sequential explicit Euler.

Finally, residual corrections and expert preferences vary across dynamical regimes. A shared predictor must accommodate small smooth corrections, sharp impulses, and recovery from earlier mistakes; a few large errors can dominate the objective and compete with accuracy on typical states. We introduce a \textbf{Regime Model Tree} with learned feature splits, leaf-specific linear residual readouts, and leaf-specific affine routing corrections. Motivated by potential gradient interference across regimes, conditioning corrections and routing on observable features encourages specialized responses to their different correction needs.

\Needspace{5\baselineskip}
Our contributions are summarized as follows:

\begin{itemize}[leftmargin=1.2em,labelsep=0.55em,topsep=3pt,itemsep=2pt,parsep=0pt]
\item We introduce a closed-loop dual-expert diffusion field, PLAS for efficient Hamiltonian propagation with affine symplectic maps, and a regime-conditioned residual/router design for regime-adaptive correction.
\item We establish conditions for physical consistency, local refinement convergence, and warm-start tracking, and derive finite-window error bounds under diffusion feedback, with empirical support.
\item Evaluations on HamiBalls-1 and HamiBalls-2 demonstrate lower aggregate and late-horizon prediction errors over 192 physical steps than the evaluated baselines. Component ablations examine numerical accuracy and solver efficiency, regional specialization, and training objectives.
\end{itemize}
\par\endgroup

\section{Related Work}
\noindent\textbf{Hamiltonian learning and structured computation.}
Hamiltonian neural networks learn energy-based dynamics \citep{greydanus2019hnn}, with dissipative extensions \citep{sosanya2022dissipative}. Neural Hamiltonian Diffusions model stochastic Hamiltonian dynamics on curved manifolds \citep{park2025nhd}. SympNets learn symplectic maps \citep{jin2020sympnets}, complementing structure-preserving integration \citep{hairer2002gni}. DEER uses Newton linearization and affine recurrences to parallelize nonlinear sequences, including HNN simulation \citep{lim2024deer}; quasi-DEER and ELK address approximate Jacobians and stability \citep{gonzalez2024scalable}. PLAS builds on this solver lineage for Hamiltonian propagation within mixed-state denoising.

\noindent\textbf{Denoising and whole-window physical prediction.}
Rectified flow and flow matching learn transport fields from interpolation paths \citep{liu2022rectified,lipman2022flow}, while diffusion transformers enable joint denoising \citep{peebles2022dit}. Hamiltonian Score Matching and Generative Flows introduce Hamiltonian velocity predictors for generative modeling \citep{holderrieth2024hsm}; port-Hamiltonian formulations interpret diffusion sampling through energy and feedback control \citep{darehmiraki2026phdiffusion}. These approaches structure evolution in generative time. For physical prediction, PhysiFormer jointly denoises future trajectory states across physical time \citep{chen2026physiformer}.

\noindent\textbf{Combining denoising with Hamiltonian dynamics.}
Denoising Hamiltonian Network (DHN) denoises overlapping local Hamiltonian blocks for prediction and physical reasoning \citep{deng2025dhn}. Hamiltonian-Guided Diffusion Fields (HG-DPF) guides joint denoising of larger prediction windows with a separately learned Hamiltonian \citep{sang2026hgdpf}. Both build on smooth Hamiltonian dynamics. HamiFormer couples whole-window denoising with residual-corrected propagation through routed feedback, accommodating non-Hamiltonian effects and prediction errors.

\section{HamiFormer}
\label{sec:method}
At each denoising step, HamiFormer mixes diffusion and Hamiltonian candidates across the prediction window, feeding each mixed state into subsequent physical propagation (Figure~\ref{fig:architecture}).

\begin{figure}[t]
\centering
\includegraphics[width=\textwidth]{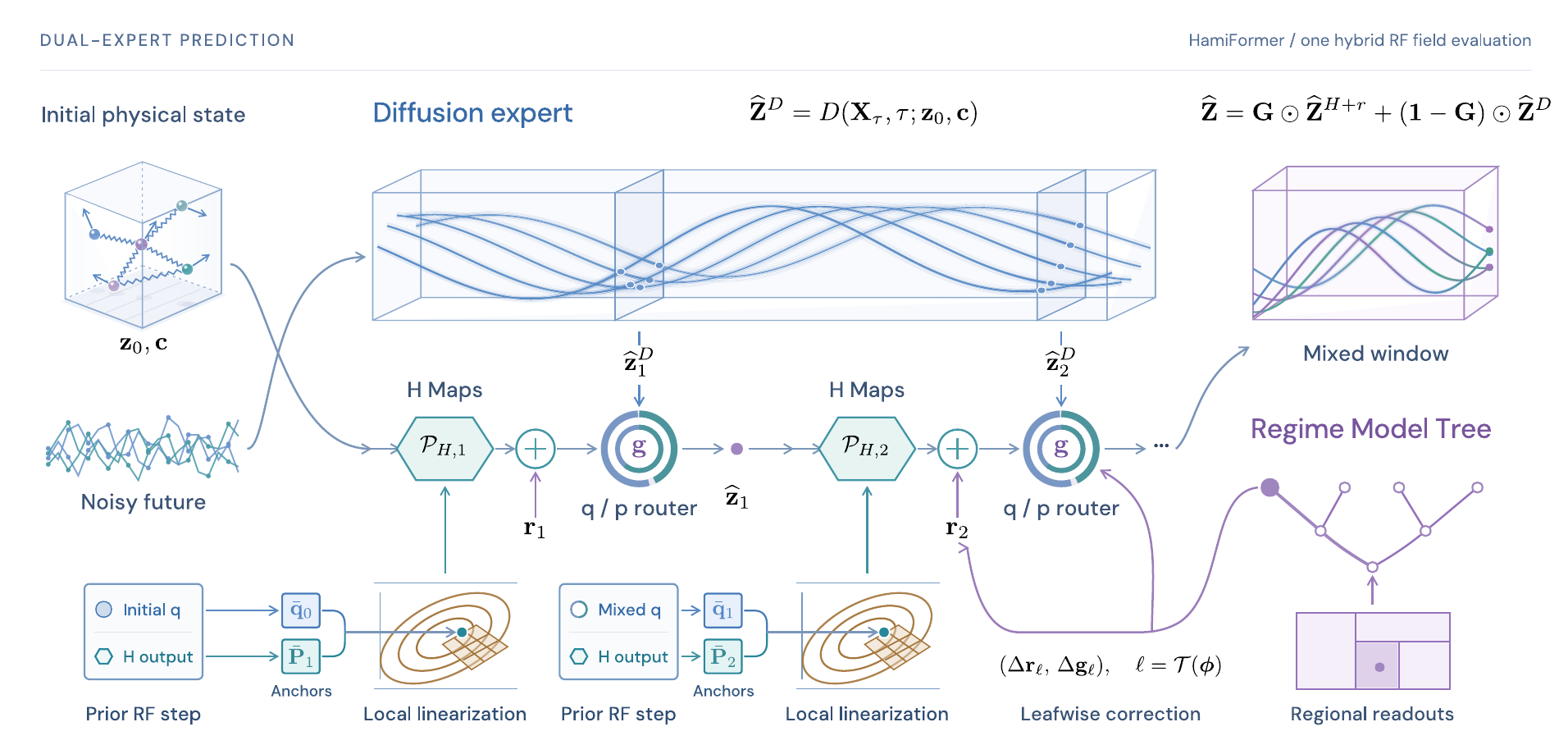}
\caption{Overview of HamiFormer. At each RF field evaluation, a whole-window diffusion expert and a residual-corrected Hamiltonian expert produce candidates for componentwise mixing. PLAS constructs affine maps in parallel from cached anchors and scans the mixed history. The Regime Model Tree conditions the residual and router readouts.}
\label{fig:architecture}
\end{figure}

\subsection{States, time axes, and the diffusion expert}
We represent physical states by positions and momenta, providing a phase-space description for Hamiltonian propagation. Given initial state $\mathbf z_0$ and attributes $\mathbf c$, we predict $L$ states $\mathbf Z=(\mathbf z_1,\ldots,\mathbf z_L)$, where $\mathbf z_k=(\mathbf q_k,\mathbf p_k)$ collects the positions and momenta of $n$ objects. Physical time is $t_k=kh$; RF time is $\tau\in[0,1]$. We use $N$ RF intervals and $K$ PLAS refinements per field evaluation.

Let $\mathbf X$ be the trajectory in fixed normalized coordinates. The diffusion expert $D$ reads the full noisy future block, $\tau$, $\mathbf z_0$, and $\mathbf c$, and returns a clean estimate and object--time tokens. We train it using
\begin{equation}
 \mathbf X_\tau=\tau\mathbf X+(1-\tau)\vepsilon,\quad
 \vepsilon\sim\mathcal N(\mathbf 0,\Id),\quad
 \tau=\sigma(\xi),\quad \xi\sim\mathcal N(\mu_\tau,\sigma_\tau^2),
 \label{eq:path}
\end{equation}
with the velocity loss (where $d_\tau=\max(1-\tau,\varepsilon_\tau)$)
\begin{equation}
 \mathcal L_D=\E\left\Vert
 \frac{D(\mathbf X_\tau,\tau;\mathbf z_0,\mathbf c)-\mathbf X_\tau}{d_\tau}
 -\frac{\mathbf X-\mathbf X_\tau}{d_\tau}\right\Vert^2=\E\left\Vert
 \frac{D(\mathbf X_\tau,\tau;\mathbf z_0,\mathbf c)-\mathbf X}{d_\tau}\right\Vert^2.
 \label{eq:d-loss}
\end{equation}
The positive denominator floor $\varepsilon_\tau$ is part of the implemented objective. The time-distribution settings and floor are specified in Appendix~\ref{app:training}. The backbone is a factorized diffusion transformer; architectural details are in Appendix~\ref{app:implementation}. Its clean physical-coordinate prediction is $\widehat{\mathbf Z}^{D}$.

\subsection{Dual experts and mixed-state feedback}
\label{sec:dual-expert}
At each RF field evaluation, diffusion supplies a whole-window candidate $\widehat{\mathbf Z}^{D}$.
The Hamiltonian operator $\mathcal P_{H,k}$ produces a structured one-step candidate corrected by
residual $\mathbf r_k$. Router weights $\mathbf g_{k,i}\in[0,1]^2$ act on its position and momentum
components, broadcast across coordinates:
\begin{equation}
 \begin{gathered}
 \widehat{\mathbf z}^{H}_k=\mathcal P_{H,k}(\widehat{\mathbf z}_{k-1}),\quad
 \widehat{\mathbf z}^{\HR}_k=\widehat{\mathbf z}^{H}_k+\mathbf r_k,\\
 \widehat{\mathbf z}_{k,i}=\mathbf g_{k,i}\odot\widehat{\mathbf z}^{\HR}_{k,i}
 +(\mathbf 1-\mathbf g_{k,i})\odot\widehat{\mathbf z}^{D}_{k,i}.
 \end{gathered}
 \label{eq:mix}
\end{equation}
The propagation operator acts on the \emph{current mixed predecessor}. The corrected $H+r$ candidate is formed before the router evaluates it. Crucially, the mixed output $\widehat{\mathbf z}_k$ immediately becomes the propagation input at the next physical edge. Thus diffusion-based corrections can alter all subsequent structured propagation. The following subsections detail $\mathcal P_H$, $\mathbf r$, and $\mathbf g$.

\begingroup
\setlength{\abovedisplayskip}{13pt plus 1pt minus 1pt}
\setlength{\belowdisplayskip}{13pt plus 1pt minus 1pt}
\setlength{\abovedisplayshortskip}{9pt plus 1pt minus 1pt}
\setlength{\belowdisplayshortskip}{9pt plus 1pt minus 1pt}
\subsection{Hamiltonian propagation via Parallel Local Affine Scan}
\label{sec:plas}
PLAS-SymEuler instantiates $\mathcal P_H$ through Newton-type trajectory linearization \citep{lim2024deer}. A scalar $H_\theta(\mathbf q,\mathbf p;\mathbf c)$ defines the Hamiltonian prior. Suppressing $\theta,\mathbf c$, symplectic Euler gives
\begin{equation}
 \mathbf p=\mathbf P+hH_{\mathbf q}(\mathbf q,\mathbf P),\qquad \mathbf Q=\mathbf q+hH_{\mathbf p}(\mathbf q,\mathbf P).
 \label{eq:se}
\end{equation}
A mixed-coordinate anchor $(\bar{\mathbf q},\bar{\mathbf P})$ pairs predecessor position and next-step Hamiltonian momentum. Parallel gradient and Hessian evaluation at all $L$ fixed anchors defines full-system blocks
\begin{equation}
 \mathbf U=hH_{\mathbf q\mathbf q}(\bar{\mathbf q},\bar{\mathbf P}),\quad
 \mathbf B=\Id+hH_{\mathbf q\mathbf p}(\bar{\mathbf q},\bar{\mathbf P}),\quad
 \mathbf V=hH_{\mathbf p\mathbf p}(\bar{\mathbf q},\bar{\mathbf P}).
 \label{eq:jet}
\end{equation}
Linearizing Eq.~\eqref{eq:se} at the anchor yields an affine map $\mathcal P_{H,k}(\mathbf z)=\mathbf A_k\mathbf z+\mathbf b_k$. For nonsingular $\mathbf B$, its offset is $\mathbf b=\mathbf z_{a,+}-\mathbf A\mathbf z_a$, with linear part and anchor-compatible source--target pair
\setlength{\abovedisplayskip}{16pt plus 1pt minus 1pt}
\setlength{\belowdisplayskip}{16pt plus 1pt minus 1pt}
\begin{equation}
 \mathbf A=\begin{bmatrix}
 \mathbf B^\top-\mathbf V\mathbf B^{-1}\mathbf U & \mathbf V\mathbf B^{-1}\\
 -\mathbf B^{-1}\mathbf U & \mathbf B^{-1}
 \end{bmatrix},\quad
 \begin{aligned}
 \mathbf z_a&=[\bar{\mathbf q},\,\bar{\mathbf P}+hH_{\mathbf q}(\bar{\mathbf q},\bar{\mathbf P})]^\top\\
\mathbf z_{a,+}&=[\bar{\mathbf q}+hH_{\mathbf p}(\bar{\mathbf q},\bar{\mathbf P}),\,\bar{\mathbf P}]^\top
 \end{aligned}.
 \label{eq:affine}
\end{equation}
Affine maps act sequentially on newly mixed predecessors; residuals and routing remain nonlinear, and $\mathbf B^{-1}$ is applied through linear solves.

We use one warm-started refinement ($K=1$). Each committed rollout supplies the next anchors: mixed predecessor positions $\mathbf q_{k-1}$ and uncorrected H-candidate momenta $\mathbf P_k$. Carrying these pairs across RF intervals preserves refinement progress as the denoising conditions evolve.
\par\endgroup

\begingroup
\setlength{\abovedisplayskip}{9pt plus 1pt minus 1pt}
\setlength{\belowdisplayskip}{9pt plus 1pt minus 1pt}
\setlength{\abovedisplayshortskip}{9pt plus 1pt minus 1pt}
\setlength{\belowdisplayshortskip}{9pt plus 1pt minus 1pt}
\subsection{Regime-conditioned residuals and routing}
\label{sec:tree}
A shared residual $\mathbf r_0$ corrects $\mathbf q$ and $\mathbf p$ using diffusion tokens, noisy and initial states, mixed history, candidates, and attributes. Observable features encode candidate differences, geometry, and history. A CART-based model tree \citep{breiman1984cart,quinlan1992continuous} selects a regional linear readout:
\begin{equation}
 \ell=\mathcal T(\vphi),\qquad \Delta\mathbf r_\ell=\mathbf W_\ell\vpsi.
 \label{eq:leaf-readout}
\end{equation}
The split features and thresholds are fitted on the training split. The readout vector $\vpsi$ contains standardized, clipped features and an intercept. A hierarchical ridge penalty shares information through a common linear block and leaf-specific deviations. The effective residual in Eq.~\eqref{eq:mix} is
\begin{equation}
 \mathbf r_k=(\mathbf 1-\valpha)\odot\mathbf r_{0,k}+\valpha\odot\Delta\mathbf r_{\ell,k},
 \label{eq:residual}
\end{equation}
where $\valpha=(\alpha_q,\alpha_p)\in[0,1]^2$ is fitted on the training split, and each coefficient is broadcast across its spatial coordinates. The same blend applies on the first edge, where the leaf readout is zero.

The router uses a physical-time GRU \citep{cho2014encoder}, a shared scalar logit, and nonlinear $\mathbf q/\mathbf p$ offsets. Its leaf-specific affine correction reads current causal observables $\mathbf o$ and an intercept:
\begin{equation}
 \mathbf g_{k,i}=\sigma\!\left(\logit(\mathbf g^{\rm base}_{k,i})
       +\valpha\odot\mathbf C_\ell[\mathbf o_{k,i};1]\right).
 \label{eq:tree-gate}
\end{equation}
All features are available at the current causal stage. Router recurrence resets for each complete physical rollout, while mixed-coordinate anchors persist between RF intervals. Appendix~\ref{app:implementation} specifies the features, normalization, and tree dimensions for residuals and routing.
\par\endgroup

\begingroup
\linespread{1.025}\selectfont
\setlength{\parskip}{10pt minus 1.5pt}
\subsection{Stateful RF sampling and staged learning}
\label{sec:sample-train}
\textbf{Stateful sampling.} Each field evaluation returns $(\widehat{\mathbf X}-\mathbf X_\tau)/d_\tau$ and updated mixed-coordinate anchors. We use $N=20$ RF intervals with Heun updates and a final Euler step, initializing with two diffusion-only intervals. Both Heun evaluations share the incoming anchor; the left evaluation commits the next anchor (Algorithm~\ref{alg:sampler}).

\textbf{Joint base learning.} We jointly train $D$ and $H_\theta$ for 50,000 updates under $\mathcal L_{\rm base}=\mathcal L_D+\mathcal L_H$. The diffusion loss in Eq.~\eqref{eq:d-loss} supervises whole-window denoising. The Hamiltonian loss fits local discrete Hamiltonian relations between adjacent states, with robust object--edge weights that reduce the influence of large relation errors. Both networks are then frozen during residual and router learning.

\textbf{Residual and router learning.} The 800 downstream updates comprise four 200-update stages. The shared residual first learns local corrections and closed-loop recovery through local replay of trajectories generated by randomly mixing the expert candidates. Two scalar-router stages then learn recurrent mixing on trajectories from the preceding policy, held fixed within each stage: the first minimizes recurrent projection regret (Appendix~\ref{app:gate-loss}); the second adds componentwise non-degradation relative to diffusion.

\textbf{Regional specialization.} We then fit the model tree on training trajectories and freeze its splits. The final stage freezes the scalar trunk and trains position/momentum heads and leaf-affine routing corrections using candidate preference, recurrent projection regret, and non-degradation objectives. Hierarchical ridge fitting updates leaf residuals from refreshed rollouts; blend coefficients are fitted on the training split. Appendix~\ref{app:training} details all objectives, fitting procedures, and refresh schedules.
\par\endgroup

\begingroup
\linespread{1.025}\selectfont
\setlength{\parskip}{10pt minus 1.5pt}
\subsection{Theoretical properties}
\label{sec:theory}
We analyze PLAS propagation, single-refinement tracking, and finite-window error propagation. Symplecticity is established at the PLAS-map level, and smoothness assumptions apply to the learned Hamiltonian. Physical collision consistency is analyzed separately. Proofs are given in Appendix~\ref{app:proofs}. Appendix~\ref{app:theory-evidence} evaluates complete deployed predictions, including tree-boundary cases.

\label{sec:theory-structure}
\begin{theorem}[Symplecticity and local accuracy]
\label{thm:affine}
For a learned Hamiltonian $H\in C^2$, fixed anchors and attributes, and invertible $\mathbf B$ in Eq.~\eqref{eq:jet}, PLAS-SymEuler satisfies $\mathbf A^\top\mathbf J\mathbf A=\mathbf J$, where $\mathbf J$ is canonical. For $H\in C^3$, the local symplectic-Euler map $F_h$ through $\mathbf z_a$ obeys
\begin{equation}
 F_h(\mathbf z)=\mathbf A\mathbf z+\mathbf b+O(\norm{\mathbf z-\mathbf z_a}^2).
 \label{eq:local-remainder-summary}
\end{equation}
For quadratic $H$, the affine map is exact wherever the implicit step is defined.
\end{theorem}

\begin{theorem}[Self-consistency and quadratic refinement]
\label{thm:mixed-refinement}
Fix RF conditions $\mathbf y=(\mathbf X_\tau,\tau)$, diffusion outputs, and scan initialization. Let $\mathcal R_{\mathbf y}$ and $\mathcal C_{\mathbf y}$ return the next anchor and clean trajectory. For a finite rollout with $H\in C^3$ satisfying the local regularity conditions in Appendix~\ref{app:mixed-proof}, the unique branch-consistent fixed point $\mathbf a^*$ satisfies, locally for some $C_L<\infty$,
\begin{equation}
 \norm{\mathcal R_{\mathbf y}(\mathbf a)-\mathbf a^*}
 +\norm{\mathcal C_{\mathbf y}(\mathbf a)-\mathcal C_{\mathbf y}(\mathbf a^*)}
 \leq C_L\norm{\mathbf a-\mathbf a^*}^2.
 \label{eq:mixed-quadratic-summary}
\end{equation}
This one-update bound yields quadratic local convergence and, under Appendix~\ref{app:tracking-proof}'s conditions, local $K=1$ warm-start tracking. Fixed-point H candidates obey Eq.~\eqref{eq:se} at mixed predecessors.
\end{theorem}
\par\endgroup

\label{sec:theory-reset}
\begin{proposition}[Finite-window error bound]
\label{prop:reset}
For the nonnegative physical--memory recurrence in Appendix~\ref{app:reset-proof} on an $L$-edge window with diffusion error at most $U_D$, the physical error $\mathbf e_k$ satisfies
\begin{equation}
 \max_{k\leq L}\norm{\mathbf e_k}_\infty\leq V_L+C_D(L)U_D.
 \label{eq:reset-bound-summary}
\end{equation}
Here $C_D(L)=\max_{k\leq L}\sum_{j=1}^k\norm{\mathbf K_{k,j}}_\infty$ and $V_L=\max_{k\leq L}\norm{\mathbf v_k}_\infty$. In the equality recurrence, $\mathbf K_{k,j}$ maps diffusion error at $j$ to physical error at $k$; $\mathbf v_k$ is the physical response with zero diffusion forcing, retaining initial and non-diffusion errors.
\end{proposition}
Appendix~\ref{app:finite-rf} connects finite refinement to RF outputs. Deployed single refinement matches the prediction accuracy of an additional-refinement reference (computed with 16 refinement updates), with 0.1903\% mean positive excess loss over 192 HamiBalls-1 edges (Appendix~\ref{app:theory-tails}). Appendices~\ref{app:physical-consistency} and~\ref{app:collisions} establish physical-consistency bounds for smooth propagation and matched collision events.

\section{Experiments}
\label{sec:experiments}
We evaluate full-model accuracy and error accumulation across smooth motion and contacts, then examine solver performance, mixed feedback, regional specialization, and training objectives.
\subsection{Protocol and comparison scope}
We evaluate on HamiBalls-1 and HamiBalls-2. HamiBalls-1 contains five balls in a closed square, each attached by a spring to the center. Balls interact only through inelastic collisions, which also occur with the walls; gravity is absent. HamiBalls-2 contains five to ten balls in three dimensions, sparse spring connections, gravity, and inelastic ball collisions in a bounded environment. Initial positions and momenta, mass, radius, restitution, and the HamiBalls-2 spring graph and its parameters are observable. Training uses 48-edge train-split windows; validation uses 192-edge rollouts from episode initialization. Each evaluation uses 512 physical sources from four data-generation seeds and two fixed diffusion sampling-noise realizations per source.

For $c\in\{z,q,p\}$, MSE averages normalized squared errors over valid source--noise--object--edge entries and coordinates of $c$, excluding the initial frame and padding. Contact contains contact-involved cells; Continuous contains the rest (Appendix~\ref{app:results}).

Baselines are PhysiFormer and matched-size HG-DPF on HamiBalls-1, and PhysiFormer, DiT, DHN, and Transformer-AR with one- or four-state contexts on HamiBalls-2. Transformer-AR uses teacher forcing; DiT follows the matched RF protocol. PhysiFormer, DiT, and Transformer-AR use 50,000 updates with batch size 64; model sizes are about 1.2M/5.6M for HamiBalls-1/2 (Appendix~\ref{app:implementation}). Context-four remains worse in true-predecessor single-step evaluation and after three true transitions before free rollout (Appendix~\ref{app:results}).

Whole-window models chain four 48-edge windows from predicted endpoints. DHN and Transformer-AR roll out recursively. Ground truth is supplied only at initialization. We report full-rollout and disjoint-interval errors over edges 1--48, 49--96, and 97--192 without resets, separating initial accuracy from later prediction performance.

\FloatBarrier
\subsection{Trajectory Prediction}
\label{sec:trajectory-prediction}
We first test whether HamiFormer improves 192-edge prediction on both systems. Table~\ref{tab:main} reports the lowest total MSE for Ours; position/momentum errors within Contact and Continuous subsets test whether the gains extend to contact-involved states and the remaining trajectory cells.
\begin{table}[htbp]
\centering\small
\setlength{\tabcolsep}{3pt}
\caption{Pooled 192-edge normalized MSE ($\downarrow$), with ground truth supplied only at initialization. Green, yellow, and red mark the best, second-best, and third-best results per dataset and metric.}
\label{tab:main}
\begin{tabularx}{\textwidth}{l*{7}{>{\centering\arraybackslash}X}}
\toprule
\multirow{2}{*}{Method} & \multicolumn{3}{c}{Total} & \multicolumn{2}{c}{Continuous} & \multicolumn{2}{c}{Contact}\\
\cmidrule(lr){2-4}\cmidrule(lr){5-6}\cmidrule(l){7-8}
 & $z$ & $q$ & $p$ & $q$ & $p$ & $q$ & $p$\\
\midrule
\multicolumn{8}{l}{\textit{HamiBalls-1}}\\
PhysiFormer & \resultsecond{0.34029} & \resultsecond{0.34689} & \resultsecond{0.33369} & \resultsecond{0.34710} & \resultsecond{0.32854} & \resultsecond{0.33378} & \resultsecond{0.65186}\\
HG-DPF & \resultthird{1.17966} & \resultthird{1.29854} & \resultthird{1.06077} & \resultthird{1.29673} & \resultthird{1.05897} & \resultthird{1.41008} & \resultthird{1.17228}\\
\textbf{HamiFormer (Ours)} & \resultfirst{0.25071} & \resultfirst{0.25596} & \resultfirst{0.24546} & \resultfirst{0.25613} & \resultfirst{0.23985} & \resultfirst{0.24529} & \resultfirst{0.59206}\\
\midrule
\multicolumn{8}{l}{\textit{HamiBalls-2}}\\
DHN & 0.99190 & 0.83993 & \resultsecond{1.14387} & 0.82841 & \resultsecond{1.12771} & 1.02251 & \resultsecond{1.39999}\\
Transformer-AR$_{\text{ctx}=1}$ & \resultthird{0.90340} & \resultsecond{0.56438} & 1.24241 & \resultsecond{0.55663} & 1.21678 & \resultsecond{0.68721} & 1.64851\\
Transformer-AR$_{\text{ctx}=4}$ & 1.51165 & 0.94036 & 2.08295 & 0.92730 & 2.06619 & 1.14717 & 2.34849\\
DiT & \resultsecond{0.88309} & 0.60233 & \resultthird{1.16386} & 0.59430 & \resultthird{1.14362} & \resultthird{0.72957} & \resultthird{1.48460}\\
PhysiFormer & 0.91397 & \resultthird{0.59681} & 1.23113 & \resultthird{0.58827} & 1.20712 & 0.73211 & 1.61155\\
\textbf{HamiFormer (Ours)} & \resultfirst{0.69398} & \resultfirst{0.49281} & \resultfirst{0.89516} & \resultfirst{0.48644} & \resultfirst{0.87631} & \resultfirst{0.59378} & \resultfirst{1.19375}\\
\bottomrule\end{tabularx}\end{table}

On HamiBalls-1, Ours achieves total MSE 0.25071 versus 0.34029 for PhysiFormer and 1.17966 for HG-DPF, reducing error by 26.3\% relative to PhysiFormer. On HamiBalls-2, Ours achieves 0.69398 versus DiT's 0.88309, a 21.4\% reduction over the best baseline. Figure~\ref{fig:h2-qualitative} complements these aggregate results with visual comparisons of object positions and spring configurations.

\noindent\begin{minipage}{\linewidth}
\centering
\includegraphics[width=\textwidth]{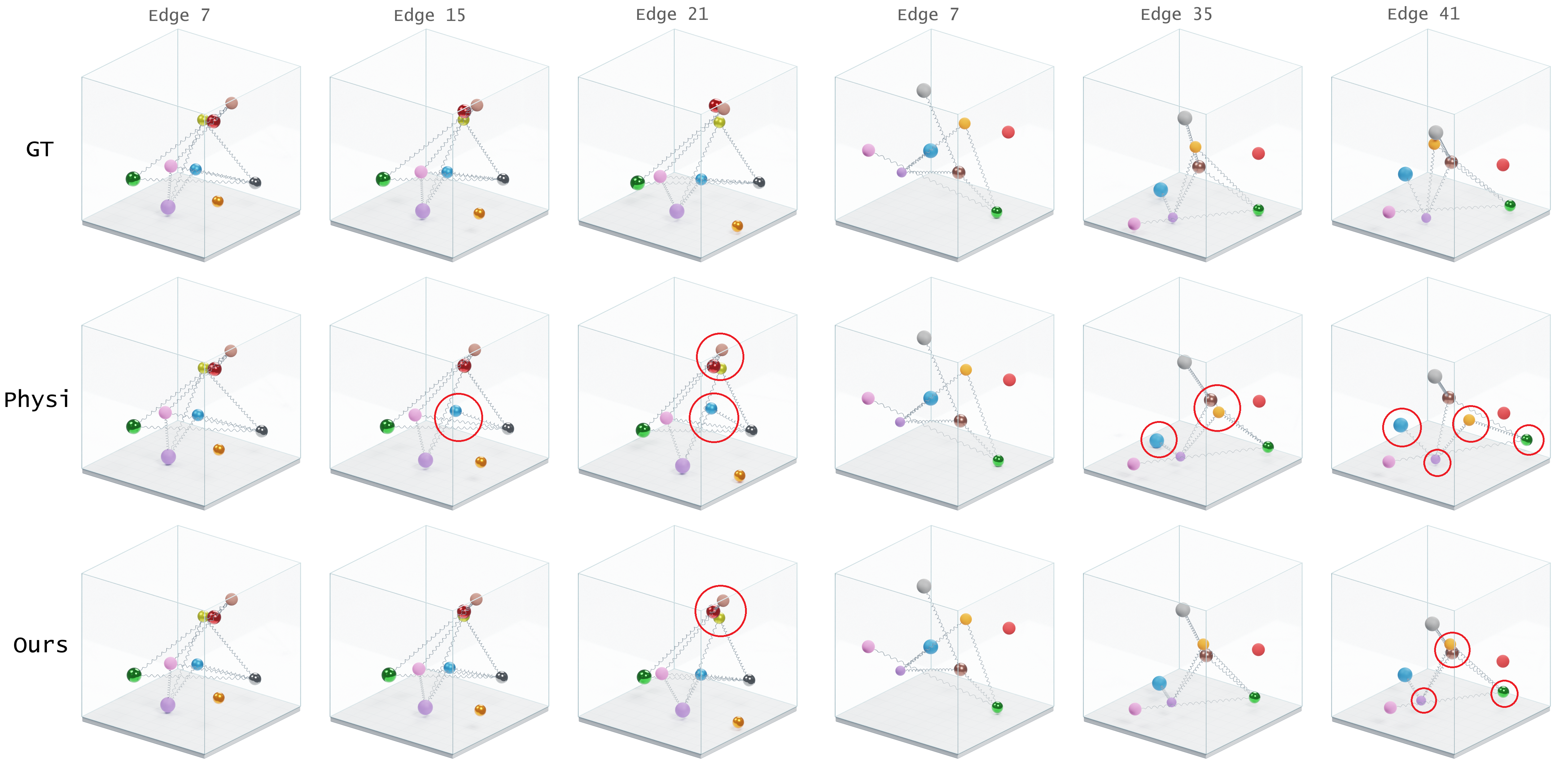}
\captionof{figure}{Two HamiBalls-2 trajectories: ground truth, PhysiFormer, and Ours at selected steps. Higher gloss denotes higher restitution; red circles mark large object-wise prediction errors.}
\label{fig:h2-qualitative}
\end{minipage}\par

\FloatBarrier
\subsection{Error Accumulation}
\label{sec:error-accumulation}
We distinguish initial from later prediction accuracy using disjoint intervals without resets (Table~\ref{tab:long}) and per-edge errors (Figure~\ref{fig:error-accumulation}). Ours leads across all HamiBalls-1 intervals; its HamiBalls-2 advantage emerges during continued prediction, which carries forward earlier errors.
\par
\begin{table}[!htbp]
\centering\small
\setlength{\tabcolsep}{2pt}
\begingroup
\setlength{\belowcaptionskip}{9pt}
\captionof{table}{Disjoint-interval normalized MSE ($\downarrow$) without resets; colors follow Table~\ref{tab:main}.}
\label{tab:long}
\endgroup
\begin{tabularx}{\textwidth}{l*{6}{>{\centering\arraybackslash}X}}
\toprule
\multirow{2}{*}{Method} & \multicolumn{2}{c}{1--48 edges} & \multicolumn{2}{c}{49--96 edges} & \multicolumn{2}{c}{97--192 edges}\\
\cmidrule(lr){2-3}\cmidrule(lr){4-5}\cmidrule(l){6-7}
 & $q$ & $p$ & $q$ & $p$ & $q$ & $p$\\
\midrule
\multicolumn{7}{l}{\textit{HamiBalls-1}}\\
PhysiFormer & \resultsecond{0.02076} & \resultsecond{0.10887} & \resultsecond{0.23468} & \resultsecond{0.29940} & \resultsecond{0.56606} & \resultsecond{0.46324}\\
HG-DPF & \resultthird{0.23083} & \resultthird{0.70319} & \resultthird{1.29546} & \resultthird{1.04699} & \resultthird{1.83393} & \resultthird{1.24646}\\
\textbf{HamiFormer (Ours)} & \resultfirst{0.01451} & \resultfirst{0.07719} & \resultfirst{0.16078} & \resultfirst{0.21496} & \resultfirst{0.42427} & \resultfirst{0.34484}\\
\midrule
\multicolumn{7}{l}{\textit{HamiBalls-2}}\\
DHN & 0.18593 & 0.72067 & 0.82700 & 1.43810 & 1.17339 & \resultsecond{1.20836}\\
Transformer-AR$_{\text{ctx}=1}$ & \resultfirst{0.02919} & \resultfirst{0.32268} & \resultfirst{0.33786} & \resultsecond{1.24024} & 0.94524 & 1.70335\\
Transformer-AR$_{\text{ctx}=4}$ & 0.16930 & 1.04295 & 0.87638 & 2.33487 & 1.35787 & 2.47699\\
DiT & 0.07969 & 0.60970 & 0.50147 & \resultthird{1.32340} & \resultsecond{0.91407} & \resultthird{1.36118}\\
PhysiFormer & \resultthird{0.06255} & \resultthird{0.54874} & \resultthird{0.46712} & 1.36271 & \resultthird{0.92879} & 1.50653\\
\textbf{HamiFormer (Ours)} & \resultsecond{0.04090} & \resultsecond{0.34356} & \resultsecond{0.35779} & \resultfirst{1.02909} & \resultfirst{0.78628} & \resultfirst{1.10399}\\
\bottomrule\end{tabularx}
\vspace{5pt}
\end{table}
\begin{figure}[!htbp]
\centering
\includegraphics[width=\textwidth]{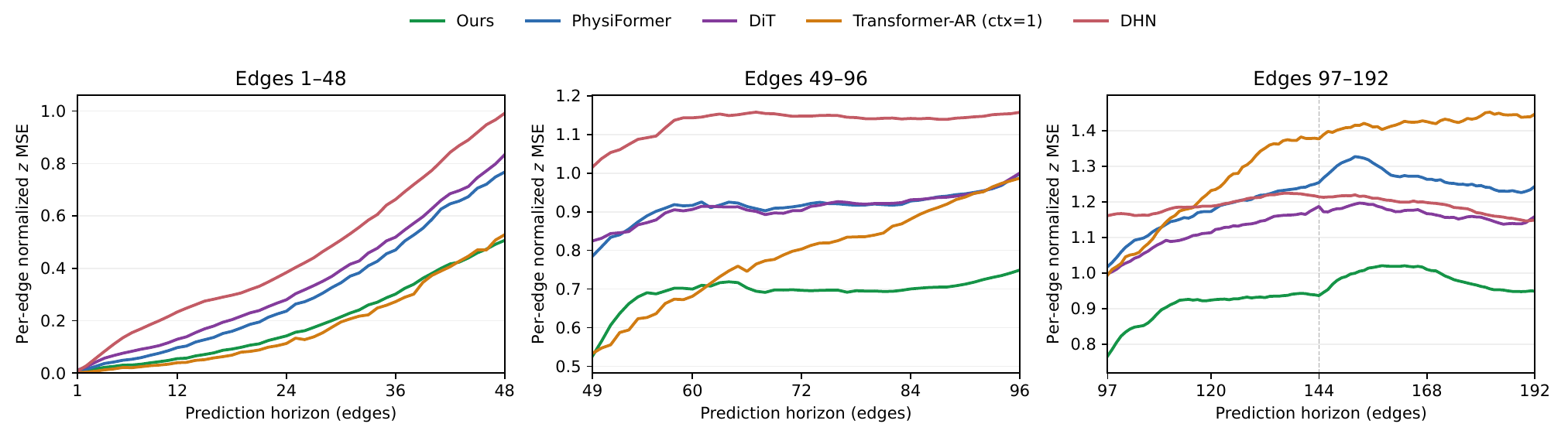}
\vspace{-5pt}
\begingroup
\normalsize
\captionof{figure}{HamiBalls-2 phase-space MSE averaged over valid objects and realizations at each edge. Axes use separate linear ranges; the dashed line marks the diffusion window boundary at edge 144. Transformer-AR is competitive initially but exhibits stronger error growth later. Ours maintains lower error across most later edges, consistent with the interval averages over the evaluated horizon.}
\label{fig:error-accumulation}
\endgroup
\vspace{-6pt}
\end{figure}

\FloatBarrier

On HamiBalls-2, context-one Transformer-AR leads initially (MSE 0.17594 versus 0.19223 for Ours), but the ordering reverses over edges 49--96 (0.69344 for Ours versus 0.78905) and 97--192 (0.94513 versus 1.32430). In the final interval, Ours also improves on DiT's 1.13762 and PhysiFormer's 1.21766, supporting sustained prediction as a source of its full-rollout advantage.

\FloatBarrier
\titlespacing*{\subsection}{0pt}{1.2ex plus 0.5ex minus 0.2ex}{0.8ex plus 0.1ex}
\subsection{Integrator ablation}
\titlespacing*{\subsection}{0pt}{2.3ex plus 0.5ex minus 0.2ex}{0.8ex plus 0.1ex}
\label{sec:integrator-ablation}
We test PLAS's implicit-step accuracy and sampling cost on HamiBalls-1 against Explicit Euler and two-iteration symplectic Euler (SymEuler2). PLAS-S sparsely refreshes Hessians; PLAS-DC adds defect correction, with frozen-jet (PLAS-DC-F) and sparse-Hessian (PLAS-DC-S) variants.

Solver errors are measured against converged FP64 Newton references at each method's inputs, using 512 sources, two noises, eight uniformly selected edges per 48-edge window, and all 18 accepted Hamiltonian-bearing RF fields. This measures implicit-step agreement, separately from ground-truth prediction MSE. Timing uses synchronized B64$\times$48 sampling with a shared interface. All variants retain the same RF schedule, model weights, residual corrections, and routing rules, with only the numerical solver changed.

\begin{figure}[!htbp]
\centering
\includegraphics[width=\textwidth]{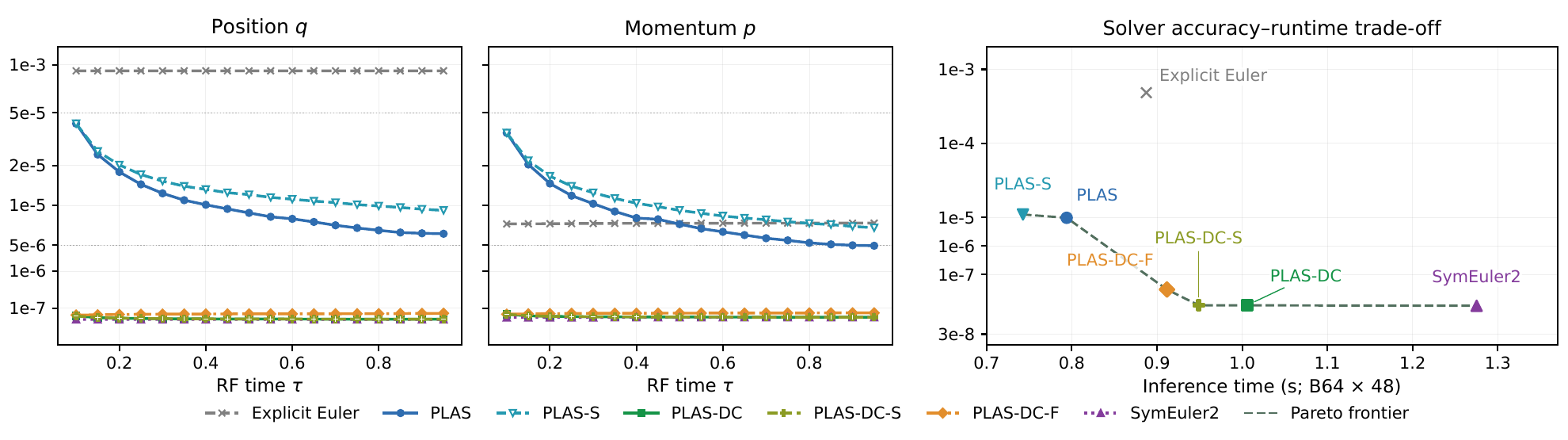}
\caption{HamiBalls-1 integrator ablation. Left/center: cumulative-mean $q/p$ RMSE through each RF field. Right: pooled phase-space RMSE versus synchronized B64$\times$48 sampling time; the dashed line connects nondominated configurations. See Table~\ref{tab:integrators} and Appendix~\ref{app:integrator-ablation} for values and protocol.}
\label{fig:integrator-ablation}
\end{figure}
\begin{table}[!htbp]
\centering\small
\caption{Solver RMSE against FP64 Newton on HamiBalls-1 (512 sources, two noises). Time: median of 11 warmed, synchronized B64$\times$48 runs with a shared comparison interface (Appendix~\ref{app:integrator-ablation}).}
\label{tab:integrators}
\begin{tabularx}{\textwidth}{l*{4}{>{\raggedleft\arraybackslash}X}}
\toprule
\multirow{2}{*}{Solver} & \multicolumn{3}{c}{RMSE ($\times10^{-6}$) $\downarrow$} & \multirow{2}{*}{Time (s) $\downarrow$}\\
\cmidrule(lr){2-4}
 & $q$ & $p$ & $z$ & \\
\midrule
Explicit Euler & 685.261 & 7.34700 & 484.580 & \resultthird{0.8875}\\
SymEuler2 & \resultfirst{0.04955} & \resultfirst{0.05661} & \resultfirst{0.05319} & 1.2752\\
PLAS & 10.6093 & 8.92038 & 9.80127 & \resultsecond{0.7940}\\
PLAS-S & 12.1360 & 9.74284 & 11.0047 & \resultfirst{0.7427}\\
PLAS-DC & \resultsecond{0.05031} & \resultsecond{0.05694} & \resultsecond{0.05373} & 1.0057\\
PLAS-DC-F & 0.07266 & 0.07510 & 0.07389 & 0.9117\\
PLAS-DC-S & \resultthird{0.05042} & \resultthird{0.05695} & \resultthird{0.05378} & 0.9490\\
\bottomrule
\end{tabularx}
\end{table}

PLAS's cumulative-mean errors decrease along RF (Figure~\ref{fig:integrator-ablation}). It achieves phase-space RMSE $9.80\times10^{-6}$ in 0.794 seconds, versus Euler's $4.85\times10^{-4}$ in 0.887 seconds. Thus PLAS attains lower solver error and shorter sampling time than Explicit Euler in this comparison. PLAS-S takes 0.743 seconds with modestly increased error; PLAS-DC-S approaches SymEuler2's accuracy in 0.949 versus 1.275 seconds (Table~\ref{tab:integrators}). The variants provide additional accuracy--runtime choices, while Ours uses full-Hessian PLAS.

PLAS achieves near-SymEuler2 prediction MSE at lower runtime than either sequential solver (Appendix~\ref{app:integrator-endpoints}). Its window-parallel derivative evaluation and lightweight affine propagation amortize refinement across RF fields while preserving the symplectic structure of its affine maps. For larger systems, PLAS-DLR reduces curvature and solve costs through diagonal--low-rank structure; Appendix~\ref{app:scaling} reports HamiBalls-1 accuracy and runtime/memory scaling with increasing object count.

\FloatBarrier

\begingroup
\setlength{\intextsep}{6pt plus 1pt minus 1pt}
\subsection{Expert feedback ablation}
\label{sec:expert-feedback}
We isolate the contribution of mixed-state feedback on 64 sources per dataset with two noises over the initial 48 edges. H+r candidate only outputs and propagates H+r; output-only mixing propagates H+r but outputs the blend; Ours propagates and outputs the blend.
\begin{table}[htbp]
\centering\small
\setlength{\tabcolsep}{4pt}
\renewcommand{\arraystretch}{1.15}
\setlength{\abovecaptionskip}{3pt}
\caption{Expert feedback ablation: normalized MSE ($\downarrow$), ranked by color within each column.}
\label{tab:expert-feedback}
\begin{tabularx}{\textwidth}{l*{6}{>{\centering\arraybackslash}X}}
\toprule
\multirow{2}{*}{Method} & \multicolumn{3}{c}{HamiBalls-1} & \multicolumn{3}{c}{HamiBalls-2} \\
\cmidrule(lr){2-4}\cmidrule(lr){5-7}
 & $q$ & $p$ & $z$ & $q$ & $p$ & $z$ \\
\midrule
H+r candidate only & \resultthird{0.022601} & \resultthird{0.098554} & \resultthird{0.060577} & \resultthird{0.391673} & \resultthird{1.069561} & \resultthird{0.730617} \\
Output-only mixing & \resultsecond{0.020802} & \resultsecond{0.084092} & \resultsecond{0.052447} & \resultsecond{0.079374} & \resultsecond{0.426752} & \resultsecond{0.253063} \\
Ours & \resultfirst{0.014472} & \resultfirst{0.077240} & \resultfirst{0.045856} & \resultfirst{0.043307} & \resultfirst{0.378203} & \resultfirst{0.210755} \\
\bottomrule
\end{tabularx}
\end{table}

Output mixing improves on H+r alone; mixed feedback further reduces phase-space MSE by 12.6\%/16.7\% on HamiBalls-1/2, improving both position and momentum (Table~\ref{tab:expert-feedback}).

Against internal D-only on 512 HamiBalls-1 sources with two sampling noises, Ours reduces phase-space MSE from 0.06629 to 0.04585 over 48 edges (30.8\%) and from 0.33863 to 0.25071 over 192 edges (26.0\%; Appendix~\ref{app:internal-d-control}). Both methods continue from their own predicted endpoints.

\titlespacing*{\subsection}{0pt}{1.2ex plus 0.5ex minus 0.2ex}{0.8ex plus 0.1ex}
\subsection{Model-tree and objective ablation}
\label{sec:tree-objective-ablation}
We cross shared/tree-conditioned residual-router structures with basic squared-error/full objectives on HamiBalls-1 (Table~\ref{tab:tree-objective}; Appendix~\ref{app:training}). Trees add regional readouts to shared $r_0$/router models; GL uses one global readout. All share frozen H/D models and 200 residual plus three 200-update router stages, evaluated on 512 sources with two noises over the initial 48 edges.
\begin{table}[htbp]
\centering\small
\setlength{\abovecaptionskip}{7pt}
\setlength{\belowcaptionskip}{9pt}
\setlength{\tabcolsep}{4pt}
\renewcommand{\arraystretch}{1.15}
\caption{HamiBalls-1, 48 edges (\%, $\uparrow$). Error reductions are relative to PhysiFormer. Cell wins compare noise-mean phase-space MSE per source--edge--object against the same baseline.}
\label{tab:tree-objective}
\begin{tabularx}{\textwidth}{clc*{4}{>{\raggedleft\arraybackslash}X}}
\toprule
\multirow{2}{*}{ID} & \multirow{2}{*}{Structure} & \multirow{2}{*}{Objective}
& \multicolumn{3}{c}{Error reduction (\%) $\uparrow$}
& \multirow{2}{*}{\shortstack{Cell win\\(\%) $\uparrow$}}\\
\cmidrule(lr){4-6}
 & & & \multicolumn{1}{c}{$z$} & \multicolumn{1}{c}{$q$} & \multicolumn{1}{c}{$p$} & \\
\midrule
A & Shared & Basic & $-55.96$ & $-61.89$ & $-54.82$ & $18.47$\\
B & Tree & Basic & \resultsecond{$11.69$} & \resultsecond{$12.02$} & \resultsecond{$11.63$} & \resultthird{$54.14$}\\
C & Shared & Full & \resultthird{$8.42$} & \resultthird{$8.83$} & \resultthird{$8.34$} & \resultsecond{$61.60$}\\
D & Tree & Full & \cellcolor{resultfirst}$\mathbf{29.27}$ & \cellcolor{resultfirst}$\mathbf{30.19}$ & \cellcolor{resultfirst}$\mathbf{29.09}$ & \cellcolor{resultfirst}$\mathbf{64.41}$\\
\midrule
GL & Global leaf & Full & $7.03$ & $8.65$ & $6.72$ & $46.63$\\
\bottomrule
\end{tabularx}
\end{table}

Tree conditioning and the full objective each improve phase-space error under either setting of the other factor. Tree+Full reduces MSE by 29.27\% relative to PhysiFormer, versus 11.69\% for Tree+Basic and 8.42\% for Shared+Full; its 64.41\% cell win rate indicates broadly distributed gains.

GL yields a 7.03\% reduction under the same update budget, supporting regional readouts; the structure contrasts also change sharing and capacity (Appendix~\ref{app:tree-objective}).
\par\endgroup

\begingroup
\titlespacing*{\section}{0pt}{1.8ex plus 0.6ex minus 0.2ex}{1.0ex plus 0.2ex}
\section{Conclusion}
HamiFormer couples whole-window diffusion and residual-corrected Hamiltonian propagation, improving full-rollout accuracy on both HamiBalls systems. Ablations support mixed feedback, regional specialization, training objectives, and PLAS's accuracy--runtime balance. The conditional analysis establishes the symplecticity of PLAS maps, local refinement and tracking results, and a finite-window error bound.

The evaluation uses simulated systems with observable physical attributes and state supervision. Future work includes broader interaction laws and observation modalities, inference of physical context from limited observations, and adaptive refinement budgets for larger systems.
\par\endgroup

\label{main:end}

\clearpage

\FloatBarrier
\addcontentsline{toc}{section}{References}
\begingroup
\small
\bibliographystyle{aaai2027}
\renewcommand{\bibfont}{\fontsize{10}{12.8}\selectfont}
\bibliography{ref}
\endgroup

\clearpage
\appendix
\numberwithin{equation}{section}
\section*{Supplementary Material}
\addcontentsline{toc}{section}{Supplementary Material}
\addtocontents{toc}{\protect\setcounter{tocdepth}{1}}
\section{Proofs and Scope of the Scientific Claims}
\label{app:proofs}
Observable conditioning and model parameters are fixed throughout each physical rollout. The analysis uses exact arithmetic, with canonical symplectic matrix $\mathbf J=\begin{bmatrix}\mathbf 0&\Id\\-\Id&\mathbf 0\end{bmatrix}$. Norms, derivative bounds, and implicit branches are fixed on neighborhoods containing all compared points and connecting segments. The local refinement analysis is complemented by finite-precision and boundary-tail measurements in Appendix~\ref{app:theory-evidence}.

\subsection{Symplecticity and local accuracy of PLAS maps}
\label{app:affine-proof}
\noindent\textbf{Full statement of Theorem~\ref{thm:affine}.} For $H\in C^2$, fixed anchors and attributes, and invertible $\mathbf B$ in Eq.~\eqref{eq:jet}, PLAS-SymEuler is symplectic: $\mathbf A^\top\mathbf J\mathbf A=\mathbf J$, with $\mathbf J$ the canonical symplectic matrix. If $H\in C^3$, the local symplectic-Euler map $F_h$ through $\mathbf z_a$ satisfies
\begin{equation}
 F_h(\mathbf z)=\mathbf A\mathbf z+\mathbf b+O(\norm{\mathbf z-\mathbf z_a}^2).
 \label{eq:local-remainder}
\end{equation}
For quadratic $H$, the affine representation is exact wherever the implicit step is defined.

\begin{proof}[Proof of Theorem~\ref{thm:affine}]
Schwarz symmetry gives $\mathbf U^\top=\mathbf U$ and $\mathbf V^\top=\mathbf V$. Factor the matrix as
\begin{equation}
 \mathbf A=
 \underbrace{\begin{bmatrix}\Id&\mathbf V\\\mathbf 0&\Id\end{bmatrix}}_{\mathbf S_{\mathbf V}}
 \underbrace{\begin{bmatrix}\mathbf B^\top&\mathbf 0\\\mathbf 0&\mathbf B^{-1}\end{bmatrix}}_{\mathbf D_{\mathbf B}}
 \underbrace{\begin{bmatrix}\Id&\mathbf 0\\-\mathbf U&\Id\end{bmatrix}}_{\mathbf S_{-\mathbf U}}.
 \label{eq:symplectic-factorization}
\end{equation}
Direct block multiplication gives $\mathbf S_{\mathbf V}^\top\mathbf J\mathbf S_{\mathbf V}=\mathbf J$ and $\mathbf S_{-\mathbf U}^\top\mathbf J\mathbf S_{-\mathbf U}=\mathbf J$ because $\mathbf U,\mathbf V$ are symmetric. It also gives $\mathbf D_{\mathbf B}^\top\mathbf J\mathbf D_{\mathbf B}=\mathbf J$ because its diagonal blocks are inverse transposes of each other. The product is therefore symplectic. A constant translation has identity derivative, so the affine map retains this property.

For the approximation claim, consider the equation
\begin{equation}
 \boldsymbol{\Phi}(\mathbf q,\mathbf p,\mathbf P)=\mathbf P+hH_{\mathbf q}(\mathbf q,\mathbf P)-\mathbf p=\mathbf 0.
\end{equation}
At $(\mathbf q,\mathbf p)=\mathbf z_a$, $\mathbf P=\bar{\mathbf P}$ is a solution and $\partial_{\mathbf P}\boldsymbol{\Phi}=\mathbf B$ is invertible. The implicit function theorem provides a local branch $\mathbf P=\mathbf P(\mathbf q,\mathbf p)$; with $H\in C^3$, this branch is $C^2$. Define $\mathbf Q=\mathbf q+hH_{\mathbf p}(\mathbf q,\mathbf P(\mathbf q,\mathbf p))$. At the anchor-compatible source,
\begin{equation}
 d\mathbf P=\mathbf B^{-1}(d\mathbf p-\mathbf U\,d\mathbf q),\qquad
 d\mathbf Q=\mathbf B^\top d\mathbf q+\mathbf V\,d\mathbf P.
\end{equation}
Thus $DF_h(\mathbf z_a)=\mathbf A$ and $F_h(\mathbf z_a)=\mathbf z_{a,+}$. Taylor's theorem on a sufficiently small convex neighborhood yields
\begin{equation}
 F_h(\mathbf z)-\mathbf A\mathbf z-\mathbf b
 =\int_0^1(1-t)D^2F_h(\mathbf z_a+t(\mathbf z-\mathbf z_a))[\mathbf z-\mathbf z_a,\mathbf z-\mathbf z_a]dt,
\end{equation}
and hence a remainder bounded by $\tfrac12\sup\norm{D^2F_h}\norm{\mathbf z-\mathbf z_a}^2$. If $H$ is quadratic, Eq.~\eqref{eq:se} is affine in $(\mathbf q,\mathbf p,\mathbf Q,\mathbf P)$, its Hessian blocks are constant, and the solved map equals $\mathbf A\mathbf z+\mathbf b$ exactly.

Finally, for $\mathbf z=\mathbf S\mathbf x+\vmu$ with constant nonsingular $\mathbf S$, the transformed linear part is $\mathbf A_{\mathbf x}=\mathbf S^{-1}\mathbf A\mathbf S$. Consequently,
\begin{equation}
 \mathbf A_{\mathbf x}^\top\mathbf J_{\mathbf x}\mathbf A_{\mathbf x}
 =\mathbf S^\top\mathbf A^\top\mathbf J\mathbf A\mathbf S=\mathbf S^\top\mathbf J\mathbf S=\mathbf J_{\mathbf x},\qquad \mathbf J_{\mathbf x}=\mathbf S^\top\mathbf J\mathbf S.
\end{equation}
The transformed translation again leaves the identity unchanged.
\end{proof}

\paragraph{Relation to physical-step accuracy.}
The theorem's remainder is second order in \emph{anchor mismatch}. More specifically, if second and third derivatives of $H$ and the implicit inverses are uniformly bounded for $0<h\leq h_0$ on a common neighborhood, differentiation of the implicit relation twice gives $\norm{D^2F_h}\leq C h$. The local affine error is then $O(h\norm{\mathbf z-\mathbf z_a}^2)$. For a smooth Hamiltonian vector field, symplectic Euler has one-step error $O(h^2)$ relative to the continuous flow \citep{hairer2002gni}; stability bounds propagate these local errors along the trajectory. For quadratic $H$, the affine representation exactly reproduces the discrete symplectic-Euler map.

\paragraph{Role of the frozen anchor.}
The symplectic identity applies to the derivative of each affine propagation map with its anchor held fixed. For a state-dependent anchor, differentiating $\mathbf A(\mathbf a(\mathbf z))\mathbf z+\mathbf b(\mathbf a(\mathbf z))$ includes the anchor's derivative terms. The complete predictor also includes residual and routing derivatives. Equation~\eqref{eq:symplectic-factorization} characterizes the fixed-anchor Hamiltonian propagation component within this predictor.

\paragraph{Solves and finite precision.}
Invertibility of $\mathbf B$ defines the affine solve, and $\norm{\mathbf B^{-1}}$ controls its sensitivity to derivative errors and the local approximation constants. The algebraic proof establishes the exact symplectic identity; the finite-precision residual $\norm{\mathbf A^\top\mathbf J\mathbf A-\mathbf J}$ measures its numerical accuracy.

\subsection{Mixed fixed points and local quadratic refinement}
\label{app:mixed-proof}
Fix RF conditions $\mathbf y=(\mathbf X_\tau,\tau)$ and diffusion outputs. Let $\mathcal R_{\mathbf y}(\mathbf a)$ and $\mathcal C_{\mathbf y}(\mathbf a)$ return the next anchor and clean trajectory, resetting the initial state and memory each scan. A rollout is \emph{regular} if its implicit H branches are unique and nonsingular, and its residual/router/memory maps are locally Lipschitz in candidates and causal history, depending on anchors only through H candidates.

\noindent\textbf{Full statement of Theorem~\ref{thm:mixed-refinement}.} For a regular finite-horizon rollout with $H\in C^3$, $\mathcal R_{\mathbf y}$ has a unique fixed point $\mathbf a^*$ on the prescribed branches. Its H candidates satisfy Eq.~\eqref{eq:se} at the mixed predecessors. Locally, for some $C_L<\infty$,
\begin{equation}
 \norm{\mathcal R_{\mathbf y}(\mathbf a)-\mathbf a^*}
 +\norm{\mathcal C_{\mathbf y}(\mathbf a)-\mathcal C_{\mathbf y}(\mathbf a^*)}
 \leq C_L\norm{\mathbf a-\mathbf a^*}^2,
 \label{eq:mixed-quadratic}
\end{equation}
so sufficiently close anchors converge quadratically under repeated refinement.

\paragraph{Anchor, readout, and regularity.}
Fix $\mathbf y=(\mathbf X_\tau,\tau)$, diffusion predictions and tokens, observable attributes, and model parameters. Write the uncorrected Hamiltonian candidate as $(\mathbf Q_k,\mathbf P_k)$ and the mixed output as $(\mathbf q_k,\mathbf p_k)$. The scan returns
\begin{equation}
 \mathcal R_{\mathbf y}(\mathbf a)=\bigl((\mathbf q_{k-1})_{k=2}^{L},(\mathbf P_k)_{k=1}^{L}\bigr),
 \qquad
 \mathcal C_{\mathbf y}(\mathbf a)=\bigl((\mathbf q_k,\mathbf p_k)\bigr)_{k=1}^{L}.
 \label{eq:refinement-state}
\end{equation}
The initial $\mathbf q_0$ is fixed. Each scan initializes physical-time memory identically and follows Eq.~\eqref{eq:mix}. An exact mixed rollout evaluates the prescribed implicit H branches sequentially, applying the same deterministic residual, router, and memory updates after each H candidate.

For $H\in C^3$ and nonsingular implicit blocks, the implicit function theorem supplies smooth H maps. Smaller closed neighborhoods and the finite horizon give common derivative and inverse bounds. Residual/router/memory maps, including features and selectors, must be locally Lipschitz: neighborhoods stay in one continuity region, or adjacent formulas share a Lipschitz bound.

In Algorithm~\ref{alg:clean}, anchors enter only the H construction; features and GRU updates read candidates, mixed predecessors, and causal history, with memory reset each scan. Recurrent dependence is therefore mediated by the H candidates and included in the analysis. With fixed attributes and masks, the Hamiltonian networks' SiLU, softmax, and positive-offset normalization are smooth, including at physical contacts.

\begin{proof}[Proof of Theorem~\ref{thm:mixed-refinement}]
At edge $k$, put $\bar{\mathbf v}_k=(\bar{\mathbf q}_{k-1},\bar{\mathbf P}_k)$ and
$\mathbf v_k=(\mathbf q_{k-1},\mathbf P_k)$. The affine construction solves
\begin{align}
 \mathbf p_{k-1}
 &=\mathbf P_k+h\bigl[H_{\mathbf q}(\bar{\mathbf v}_k)+H_{\mathbf q\mathbf q}(\bar{\mathbf v}_k)(\mathbf q_{k-1}-\bar{\mathbf q}_{k-1})
                         +H_{\mathbf q\mathbf p}(\bar{\mathbf v}_k)(\mathbf P_k-\bar{\mathbf P}_k)\bigr],\\
 \mathbf Q_k
 &=\mathbf q_{k-1}+h\bigl[H_{\mathbf p}(\bar{\mathbf v}_k)+H_{\mathbf p\mathbf q}(\bar{\mathbf v}_k)(\mathbf q_{k-1}-\bar{\mathbf q}_{k-1})
                         +H_{\mathbf p\mathbf p}(\bar{\mathbf v}_k)(\mathbf P_k-\bar{\mathbf P}_k)\bigr].
 \label{eq:mixed-linearized}
\end{align}
For $\mathbf f=(H_{\mathbf q},H_{\mathbf p})$, suppose $D\mathbf f$ is $M_H$-Lipschitz on the segment from $\bar{\mathbf v}_k$ to $\mathbf v_k$. Taylor's integral remainder bounds the implicit-equation defect
\begin{equation}
 \mathbf d_k=
 \begin{bmatrix}\mathbf P_k+hH_{\mathbf q}(\mathbf v_k)-\mathbf p_{k-1}\\\mathbf Q_k-\mathbf q_{k-1}-hH_{\mathbf p}(\mathbf v_k)\end{bmatrix},
 \qquad
 \norm{\mathbf d_k}\leq \frac{hM_H}{2}\norm{\mathbf v_k-\bar{\mathbf v}_k}^2.
 \label{eq:mixed-defect}
\end{equation}
At a fixed point of Eq.~\eqref{eq:refinement-state}, the anchor mismatch vanishes at every edge, including the fixed initial position. Thus each H candidate satisfies the exact learned-H implicit equations at the actual mixed predecessor.

Conversely, run the prescribed unique implicit branches sequentially with the same residual, router, memory reset, and initial state. Its H candidate and mixed predecessor are the target and source of the corresponding affine jet. Supplying the resulting anchor to the scan therefore reproduces the H candidate at the first edge, then the residual, gate, mixed state, and memory. Induction gives the entire rollout and a fixed point. The same induction shows uniqueness among fixed points in these branches.

For local convergence, let $F_k$ be the smooth exact implicit H map near the fixed rollout, and let $\mathbf s_k(\mathbf a)$ be the anchor-compatible source. Theorem~\ref{thm:affine} gives
\begin{equation}
 \mathbf A_k(\mathbf a)\mathbf z+\mathbf b_k(\mathbf a)
 =F_k(\mathbf s_k(\mathbf a))+DF_k(\mathbf s_k(\mathbf a))(\mathbf z-\mathbf s_k(\mathbf a)),
 \label{eq:jet-tangent}
\end{equation}
whose error relative to $F_k(\mathbf z)$ is bounded by a constant times $\norm{\mathbf z-\mathbf s_k(\mathbf a)}^2$. Set $e=\norm{\mathbf a-\mathbf a^*}$. Smoothness of the compatible source gives
$\norm{\mathbf s_k(\mathbf a)-\mathbf z_{k-1}^*}\leq c_ke$ locally.
The first predecessor and initialized memory are exact, so its H-candidate error is $O(e^2)$. Local Lipschitzness transfers this error to the residual, gate, mixed output, and updated memory.

\begingroup\emergencystretch=3em
Inductively suppose the predecessor and memory errors entering edge $k$ are $O(e^2)$. Then $\norm{\mathbf z_{k-1}-\mathbf s_k(\mathbf a)}=O(e)$, the affine remainder is $O(e^2)$, and the exact implicit map transfers the predecessor error by a bounded factor. The H-candidate, mixed-state, and memory errors at edge $k$ are consequently $O(e^2)$. Taking a maximum over the finite horizon, and using norm equivalence in its finite-dimensional state spaces, proves Eq.~\eqref{eq:mixed-quadratic}. Choosing a common neighborhood for all updates and shrinking its radius until $C_Lr<1$ makes it invariant and proves quadratic convergence.
\par\endgroup
\end{proof}

\paragraph{Why mixing preserves the local argument.}
Local Lipschitz downstream maps transfer quadratic H-candidate perturbations without a first-order anchor term. Where differentiable, Eq.~\eqref{eq:jet-tangent} gives
\begin{equation}
 D_{\mathbf a}(\mathbf A_k(\mathbf a)\mathbf z+\mathbf b_k(\mathbf a))[\mathbf u]
 =D_{\mathbf a}\mathbf A_k(\mathbf a)[\mathbf u]\,(\mathbf z-\mathbf s_k(\mathbf a)).
 \label{eq:anchor-cancellation}
\end{equation}
Indeed $F_k(\mathbf s_k(\mathbf a))$ differentiates to $DF_k(\mathbf s_k(\mathbf a))D_{\mathbf a}\mathbf s_k(\mathbf a)[\mathbf u]$, which cancels the derivative of the subtracted source term. The physical propagation matrix can have expanding directions while direct anchor forcing vanishes at consistency. All residual, router, and memory paths remain in the subsequent propagation.

\paragraph{Regularity and finite precision.}
Hard-split jumps are treated in Appendix~\ref{app:finite-rf}. A finite soft mixture of locally Lipschitz leaves, weights, and features satisfies selector regularity, but still requires regular implicit branches and local initialization. Deployment evaluation error $\epsilon_{\rm fp}$ changes the estimate to
$e^+\leq C_Le^2+\epsilon_{\rm fp}$; it permits an invariant error band when
$C_Lr^2+\epsilon_{\rm fp}\leq r$, with $\epsilon_{\rm fp}$ bounding evaluation error throughout that neighborhood. Appendix~\ref{app:theory-evidence} measures point derivatives and finite-scale perturbation responses.

\subsection{Physical consistency of the converged mixed trajectory}
\label{app:physical-consistency}
\begin{corollary}[Consistency with smooth physical dynamics]
\label{prop:physical-consistency}
Compare a consistent mixed rollout $\mathbf z_k^*$ with samples $\mathbf z_k^\dagger$ of a smooth Hamiltonian system $H^\dagger$, sharing the initial state and fixed duration $Lh=T$. On a common convex neighborhood of trajectories and anchors, assume $h$-uniform bounds on the $C^3$ norms of $H,H^\dagger$ and implicit inverses, and $\ell_\infty$ vector-field error at most $\epsilon_H$. Set $\mathbf G_k=\diag(\mathbf g_k)$ for coordinate-expanded gates $\mathbf 0\preceq\mathbf g_k\preceq\mathbf 1$. If the correction $\mathbf c_k=\mathbf G_k\mathbf r_k+(\Id-\mathbf G_k)(\widehat{\mathbf z}_k^D-\mathbf z_k^\dagger)$ satisfies $\norm{\mathbf c_k}_\infty\leq h\eta_h$ at every edge, then, with $C_T$ independent of $h$,
\begin{equation}
 \max_{k\leq L}\norm{\mathbf z_k^*-\mathbf z_k^\dagger}_\infty
 \leq C_T(h+\epsilon_H+\eta_h).
 \label{eq:physical-consistency-main}
\end{equation}
\end{corollary}

\paragraph{Uniform physical-step assumptions.}
For $0<h\leq h_0$ and $Lh=T$, the common convex neighborhood contains the true and mixed trajectories and the mixed-coordinate anchors used by the H equations. Derivatives of $H,H^\dagger$ through order three and inverses of the implicit blocks are bounded independently of $h$. The vector-field error is
$\sup_{\mathbf z}\norm{\mathbf J\nabla H(\mathbf z)-\mathbf J\nabla H^\dagger(\mathbf z)}_\infty\leq\epsilon_H$ on this neighborhood. The realized residual/diffusion correction satisfies
\begin{equation}
 \norm{\mathbf G_k\mathbf r_k+(\Id-\mathbf G_k)(\widehat{\mathbf z}_k^D-\mathbf z_k^\dagger)}_\infty
 \leq h\eta_h\quad\text{for every }k.
 \label{eq:physical-correction-consistency}
\end{equation}
The physical-consistency bound vanishes as $h,\epsilon_H,\eta_h$ tend to zero. Collision impulses need not vanish with $h$: Appendix~\ref{app:collisions} instead bounds their approximation defects.

\begin{proof}[Proof of Corollary~\ref{prop:physical-consistency}]
Let $\Phi_h^\dagger$ denote the exact physical flow. Uniform smoothness and regular implicit branches give the familiar estimates
\begin{equation}
 \norm{F_h(\mathbf z)-F_h(\mathbf z')}_\infty\leq(1+c h)\norm{\mathbf z-\mathbf z'}_\infty,\qquad
 \norm{F_h(\mathbf z)-\Phi_h^\dagger(\mathbf z)}_\infty\leq c(h^2+h\epsilon_H)
 \label{eq:physical-local}
\end{equation}
on the common neighborhood, for sufficiently small $h$. The second estimate combines symplectic Euler's local truncation error with the vector-field approximation error; bounded derivatives and implicit inverses make $c$ uniform.

At the fixed point, Theorem~\ref{thm:mixed-refinement} identifies the H candidate with $F_h(\mathbf z_{k-1}^*)$. Let
$\mathbf c_k=\mathbf G_k\mathbf r_k+(\Id-\mathbf G_k)(\widehat{\mathbf z}_k^D-\mathbf z_k^\dagger)$ and
$E_k=\norm{\mathbf z_k^*-\mathbf z_k^\dagger}_\infty$.
Subtracting the true state from the actual mixture gives
\begin{equation}
 \mathbf z_k^*-\mathbf z_k^\dagger
 =\mathbf G_k\bigl(F_h(\mathbf z_{k-1}^*)-\mathbf z_k^\dagger\bigr)+\mathbf c_k,
 \qquad
 E_k\leq(1+c h)E_{k-1}+c(h^2+h\epsilon_H)+h\eta_h.
 \label{eq:physical-recurrence}
\end{equation}
Here $\norm{\mathbf G_k}_\infty\leq1$, regardless of how the gate depends on the candidates or history. With $E_0=0$ and $Lh=T$, discrete Gronwall yields, uniformly over all prefixes,
\begin{equation}
 \norm{\mathbf Z_h^*-\mathbf Z_h^\dagger}_\infty
 \leq T e^{cT}\bigl(c h+c\epsilon_H+\eta_h\bigr)
 \leq C_T(h+\epsilon_H+\eta_h).
 \label{eq:physical-fixedpoint}
\end{equation}
For a fixed affine normalization $\mathcal N$ with linear-part norm bounded by $C_{\mathcal N}$, set
$\mathbf v_h^*=(\mathcal N(\mathbf Z_h^*)-\mathbf X_\tau)/d_\tau$ and
$\mathbf v_h^\dagger=(\mathcal N(\mathbf Z_h^\dagger)-\mathbf X_\tau)/d_\tau$.
Since $d_\tau\geq\varepsilon_\tau>0$,
\begin{equation}
 \norm{\mathbf v_h^*-\mathbf v_h^\dagger}_\infty
 \leq \frac{C_{\mathcal N}C_T}{\varepsilon_\tau}
           (h+\epsilon_H+\eta_h).
 \label{eq:physical-field}
\end{equation}
This proves the claimed regularized-field consistency.
\end{proof}

\paragraph{Interpretation of the assumptions.}
The correction condition is imposed on the realized fixed-point rollout and accounts for both experts' contributions. Zero residual and exact D give $\eta_h=0$ for every gate; $\mathbf G_k=\Id$ and zero residual also give $\eta_h=0$ independently of D. Appendix~\ref{app:theory-health-control} evaluates zero and vanishing model/correction budgets in analytic controls, complemented by finite-step implicit-residual and prediction measurements on the learned model.

For a finite refinement with equation defects $\mathbf d_k$, stable implicit inversion adds $c\norm{\mathbf d_k}$ to Eq.~\eqref{eq:physical-recurrence}. Thus, for $Lh=T$,
\begin{equation}
 \max_{k\leq L}E_k
 \leq e^{cT}\left[E_0+cT(h+\epsilon_H)
             +\sum_{k=1}^{L}\bigl(c\norm{\mathbf d_k}+\norm{\mathbf c_k}_\infty\bigr)\right].
 \label{eq:finite-physical-defect}
\end{equation}
This separates solver error, learned-H error, and correction error. Appendix~\ref{app:collisions} incorporates collision-update defects and event amplification into an event-aligned extension.

\subsection{Single-refinement and warm-start tracking}
\label{app:tracking-proof}
\label{sec:theory-tracking}
\begin{proposition}[Warm-start tracking]
\label{thm:tracking}
Let $\mathbf a^*(\mathbf y)$ be a continuous fixed-point branch. On common-radius neighborhoods, assume $\mathcal R_{\mathbf y}$ contracts distance to $\mathbf a^*(\mathbf y)$ by a uniform factor $\rho\in(0,1)$. Along committed RF conditions $\mathbf y_s$, define $\mathbf a_s=\mathcal R_{\mathbf y_s}^{K_s}(\mathbf a_{s-1})$, $K_s\geq1$, $e_s=\norm{\mathbf a_s-\mathbf a^*(\mathbf y_s)}$, and $d_s=\norm{\mathbf a^*(\mathbf y_s)-\mathbf a^*(\mathbf y_{s-1})}$. For sufficiently small $e_0$ and $\sup_s d_s$, refinements stay in these neighborhoods and
\begin{equation}
 e_s\leq\rho^{K_s}(e_{s-1}+d_s).
 \label{eq:tracking}
\end{equation}
If $\mathbf y_s\to\mathbf y_*$ within the branch domain, then $\mathbf a_s\to\mathbf a^*(\mathbf y_*)$.
\end{proposition}

\begin{proof}[Proof of Proposition~\ref{thm:tracking}]
The local radial contraction hypothesis is
\begin{equation}
 \norm{\mathcal R_{\mathbf y}(\mathbf a)-\mathbf a^*(\mathbf y)}\leq\rho\norm{\mathbf a-\mathbf a^*(\mathbf y)},
 \qquad \norm{\mathbf a-\mathbf a^*(\mathbf y)}\leq r,
 \label{eq:local-radial-contraction}
\end{equation}
with common $r>0$ and $0<\rho<1$. Choose $\delta=r\min\{\rho,1-\rho\}$. If $e_0\leq\delta$ and $\sup_s d_s\leq\delta$, then $e_{s-1}\leq\rho r$ implies $e_{s-1}+d_s\leq r$. Each refinement remains in the ball and contracts distance by $\rho$. Applying it $K_s$ times proves Eq.~\eqref{eq:tracking} and $e_s\leq\rho r$, hence invariance by induction.

Since $K_s\geq1$, for $s>s_0$ the recurrence gives
\begin{equation}
 e_s\leq\rho^{s-s_0}e_{s_0}
       +\frac{\rho}{1-\rho}\sup_{j>s_0}d_j.
\end{equation}
If $\mathbf y_s\to\mathbf y_*$ within the branch domain, continuity gives $\mathbf a^*(\mathbf y_s)\to\mathbf a^*(\mathbf y_*)$ and hence $d_s\to0$. Letting $s\to\infty$ and then $s_0\to\infty$ proves $e_s\to0$, and the triangle inequality gives $\mathbf a_s\to\mathbf a^*(\mathbf y_*)$.
\end{proof}
\paragraph{Deployed single refinement.}
Let $\mathbf a_s^*=\mathbf a^*(\mathbf y_s)$ and $\widehat{\mathbf Z}_s=\mathcal C_{\mathbf y_s}(\mathbf a_{s-1})$. For $K_s=1$, Theorem~\ref{thm:mixed-refinement} and the triangle inequality give
\[
 e_s+\norm{\widehat{\mathbf Z}_s-\mathcal C_{\mathbf y_s}(\mathbf a_s^*)}
 \leq C_s(e_{s-1}+d_s)^2,
 \qquad e_{s-1}+d_s\leq r.
\]
Here $r$ lies within the regular neighborhoods and $C_s$ is their quadratic bound. If $\sup_s C_s r\leq\rho<1$, this implies Eq.~\eqref{eq:tracking} with $K_s=1$. A bound $\epsilon_s$ on the sum of anchor and readout evaluation errors adds $\epsilon_s$ to the right-hand side. If the branch is locally $L_a$-Lipschitz, $d_s\leq L_a\norm{\mathbf y_s-\mathbf y_{s-1}}$. Thus one update per RF condition tracks the moving target across successive conditions. In Heun sampling this recurrence follows committed left anchors; each right readout uses the same incoming anchor at its own RF condition (Appendix~\ref{app:sampling}). Appendix~\ref{app:finite-rf} propagates field errors through the finite schedule and accounts for discrete transitions.

\subsection{Componentwise diffusion resets and observable memory}
\label{app:reset-proof}
\noindent\textbf{Full statement of Proposition~\ref{prop:reset}.} Fix an $L$-edge window. Let $\mathbf E_k$ collect physical and memory error bounds, with physical part $\mathbf e_k$. Assume $\mathbf E_k\preceq\mathbf M_k\mathbf E_{k-1}+\mathbf B_k\vdelta_k+\mathbf b_k$, all terms nonnegative, where $\vdelta_k$ bounds diffusion error. In the equality recurrence, let $\mathbf K_{k,j}$ map $\vdelta_j$ to $\mathbf e_k$, and $\mathbf v_k$ be the physical response with all $\vdelta_j=\mathbf0$. Set $C_D(L)=\max_{k\leq L}\sum_{j=1}^k\norm{\mathbf K_{k,j}}_\infty$ and $V_L=\max_{k\leq L}\norm{\mathbf v_k}_\infty$. If $\max_{j\leq L}\norm{\vdelta_j}_\infty\leq U_D$, then
\begin{equation}
 \max_{k\leq L}\norm{\mathbf e_k}_\infty\leq V_L+C_D(L)U_D.
 \label{eq:reset-bound-main}
\end{equation}

\paragraph{Physical-output response matrices.}
Fix the prediction window $1\leq k\leq L$ and its conditioning variables. Let $\mathbf P=[\Id\;\mathbf 0]$ select the physical part of the augmented error vector. For the equality comparison system $\mathbf x_k=\mathbf M_k\mathbf x_{k-1}+\mathbf B_k\vdelta_k+\mathbf b_k$, with $\mathbf x_0=\mathbf E_0$, define $\boldsymbol{\Pi}_{b:a}=\mathbf M_b\cdots\mathbf M_a$, with empty product $\Id$. Its diffusion-response matrices and remaining physical response are
\begin{equation}
 \mathbf K_{k,j}=\mathbf P\boldsymbol{\Pi}_{k:j+1}\mathbf B_j,\qquad
 \mathbf v_k=\mathbf P\boldsymbol{\Pi}_{k:1}\mathbf E_0+
 \sum_{j=1}^k\mathbf P\boldsymbol{\Pi}_{k:j+1}\mathbf b_j.
 \label{eq:reset-response-matrices}
\end{equation}
Here $\mathbf v_k$ retains incoming-state errors and all non-diffusion forcing, including residual/model errors, memory errors, and boundary jumps. For finite comparison matrices and forcing, the window-specific quantities
\begin{equation}
 C_D(L)=\max_{1\leq k\leq L}\sum_{j=1}^k\norm{\mathbf K_{k,j}}_\infty,
 \qquad V_L=\max_{1\leq k\leq L}\norm{\mathbf v_k}_\infty
 \label{eq:projected-reset-condition}
\end{equation}
separate the propagated diffusion-error budget from these remaining contributions. Their dependence on $L$, realized gates, and the coupled state-memory dynamics is retained.

\begin{proof}[Proof of Proposition~\ref{prop:reset}]
Nonnegativity permits repeated substitution in the assumed error recurrence, giving
\begin{equation}
 \mathbf e_k=\mathbf P\mathbf E_k\preceq\mathbf v_k+
       \sum_{j=1}^{k}\mathbf K_{k,j}\vdelta_j,\qquad 1\leq k\leq L.
 \label{eq:reset}
\end{equation}
Apply the induced sup norm, use $\norm{\vdelta_j}_\infty\leq U_D$ over this window, and maximize over its prefixes. Equation~\eqref{eq:projected-reset-condition} gives Eq.~\eqref{eq:reset-bound-main}.
\end{proof}

\paragraph{Complete state and realized gates.}
Let $\mathbf e_k=|\widehat{\mathbf z}_k-\mathbf z_k^\dagger|$ and let $\mathbf m_k$ collect errors in all propagated physical-time memory, including accumulators. For nonnegative candidate-error bounds
\begin{align}
 |\widehat{\mathbf z}_k^{\HR}-\mathbf z_k^\dagger|
 &\preceq\mathbf F_k^{zz}\mathbf e_{k-1}+\mathbf F_k^{zm}\mathbf m_{k-1}+\mathbf a_k,\\
 |\widehat{\mathbf z}_k^D-\mathbf z_k^\dagger|&\preceq\vdelta_k,\qquad
 \mathbf m_k\preceq\mathbf F_k^{mz}\mathbf e_{k-1}+\mathbf F_k^{mm}\mathbf m_{k-1}+\mathbf F_k^{mD}\vdelta_k+\mathbf d_k,
\end{align}
the complete mixed-state inequality has
\begin{equation}
 \mathbf M_k=
 \begin{bmatrix}\mathbf G_k\mathbf F_k^{zz}&\mathbf G_k\mathbf F_k^{zm}\\\mathbf F_k^{mz}&\mathbf F_k^{mm}\end{bmatrix},
 \qquad
 \mathbf B_k=\begin{bmatrix}\Id-\mathbf G_k\\\mathbf F_k^{mD}\end{bmatrix},\qquad
 \mathbf b_k=\begin{bmatrix}\mathbf G_k\mathbf a_k\\\mathbf d_k\end{bmatrix}.
 \label{eq:reset-blocks}
\end{equation}
The memory bounds include dependence on the current mixed state and gate; $\mathbf F_k^{mD}$ retains diffusion-to-memory transfer after substituting the current-state bound. Candidate modeling errors, anchor errors, and boundary jumps enter the additive terms or their full region bounds. This coordinatewise error decomposition evaluates the gate at its realized value. The off-diagonal blocks preserve cross-object and memory-to-state error transfer.

\paragraph{Feedback and finite response budgets.}
The factors $\mathbf G_k$ attenuate inherited physical-candidate contributions, while $\Id-\mathbf G_k$ admits the diffusion candidate's error. Their subsequent effects are transported by the full products in Eq.~\eqref{eq:reset-response-matrices}. This bounds the finite-window prediction error by the diffusion expert's error budget scaled by the window response, together with the propagated initial, model, and memory errors. The projection $\mathbf P$ measures the resulting physical response, including memory-to-state transfer. Finite-scale response measurements and complete-sampling memory interventions are reported in Appendix~\ref{app:reset-evidence}.

\paragraph{Scope across prediction windows.}
Within a fixed-condition field, D jointly predicts the whole window. The budget $U_D$ is conditioned on that window's incoming state and RF conditions. Chained prediction applies the same finite-window statement with the corresponding incoming-state, diffusion, and memory budgets at each window.

\section{Consistency Across Regular Collisions}
\label{app:collisions}
The smooth estimate in Corollary~\ref{prop:physical-consistency} extends by composing smooth propagation with collision updates. Errors are measured after aligning corresponding event times and matching pre- and post-event states. The comparison state includes recurrent memory and conditioning variables whenever they affect subsequent updates.

\begin{proposition}[Finite-event consistency]
\label{prop:collision-consistency}
Fix a duration $T$ and trajectories with the same ordered sequence of $M$ isolated events. On common neighborhoods, assume each smooth arc of duration $t$ transfers incoming error $E$ to at most $e^{ct}(E+Ct\xi_h)$, where $c,C\geq0$ and $\xi_h=h+\epsilon_H+\eta_h$. At event $j$, assume the exact aligned update $\Psi_j$ is $\Lambda_j$-Lipschitz and its approximation satisfies $\sup_{\mathbf x}\norm{\widetilde\Psi_j(\mathbf x)-\Psi_j(\mathbf x)}\leq\delta_j$. All compared states remain in these neighborhoods. Writing $A_j=\max\{1,\Lambda_j\}$, the aligned terminal error obeys
\begin{equation}
 E_T\leq e^{cT}\left[
 \left(\prod_{j=1}^{M}A_j\right)(E_0+CT\xi_h)
 +\sum_{j=1}^{M}\left(\prod_{i=j+1}^{M}A_i\right)\delta_j\right].
 \label{eq:collision-bound}
\end{equation}
The corresponding prefix bound holds at every aligned time. If the event count and constants are uniform, the initial error, $\xi_h$, and all event defects tending to zero imply event-aligned convergence.
\end{proposition}
\begin{proof}
At event $j$, add and subtract $\Psi_j$ evaluated at the approximate incoming state to obtain $E_j^+\leq\Lambda_j E_j^-+\delta_j$. Compose this recurrence with the assumed smooth-arc bounds. Every smooth amplification is bounded by $e^{cT}$; replacing each event factor by $A_j\geq1$ bounds all earlier error contributions by the full product. Smooth arc durations sum to at most $T$, giving Eq.~\eqref{eq:collision-bound}. Truncation gives the prefix statement, and uniform bounds give convergence.
\end{proof}
This formulation allows inelastic event updates. Regularity means isolated, uniquely resolved events with a fixed local ordering and bounded aligned sensitivity. Event alignment compares both sides of each momentum jump; on smooth portions, a timing discrepancy $\theta$ adds at most $V\theta$ to the physical-state error when the flow speed is bounded by $V$. The event defects measure accuracy of the effective mixed update relative to the physical collision law.

\subsection{Collision-update and perturbation evidence}
We evaluate the first contact onset within the first 145 edges of each eligible HamiBalls-1 source, giving 510 predecessor states and two sampling noises. Physical one-edge perturbations use three directions per state and normalized-state RMS radii $10^{-6},10^{-5},10^{-4},10^{-3}$. Event ordering changes in 0, 1, 11, and 64 of the 1,530 source--direction cases at these radii. Finite perturbations can therefore cross event boundaries; we assess their prediction cost against separately simulated perturbed targets.

For all 510 sources, one paired direction is evaluated over complete 48-edge windows at the three larger radii, retaining both signs and both noises. The unperturbed simulator replay agrees with the stored 48-edge truth to MSE $1.30\times10^{-14}$. Each perturbed prediction is evaluated against its own simulated trajectory. Table~\ref{tab:collision-tails} takes the positive part of each paired MSE increase before averaging, so improvements cannot cancel adverse changes. The unperturbed window MSE is 0.080099.
\begin{table}[htbp]
\centering
\caption{Matched-target 48-edge collision perturbations on 510 sources. Positive cost is the mean casewise positive MSE increase divided by baseline MSE; intervals resample sources with paired signs and noises retained.}
\label{tab:collision-tails}
\begin{tabular}{lrrr}
\toprule
Radius & Positive cost (\%) & 95\% interval (\%) & Maximum MSE increase \\
\midrule
$10^{-5}$ & 0.0416 & [0.0221, 0.0684] & 0.00770 \\
$10^{-4}$ & 0.2098 & [0.1528, 0.2786] & 0.02501 \\
$10^{-3}$ & 0.9998 & [0.8310, 1.1957] & 0.04176 \\
\bottomrule
\end{tabular}
\end{table}

We additionally evaluate every condition across all three directions and four radii with changed event ordering or physical paired-response gain at least 10: 139 conditions, yielding 556 signed, two-noise predictions. Matched 48-edge targets give a positive cost of 0.9113\% of this subset's baseline MSE (95\% interval [0.5001\%, 1.7037\%]); the maximum individual MSE increase is 0.04542. Within the event-order-changing subset alone, positive cost is 1.2491\% [0.7046\%, 2.4421\%]. These conditional tail measurements quantify effects outside the unchanged-order comparisons without extending Proposition~\ref{prop:collision-consistency} to arbitrary event sequences.

\section{Empirical Checks of the Theoretical Mechanisms}
\label{app:theory-evidence}
These HamiBalls-1 checks connect local solver properties to deployed $K=1$ sampling: implicit regularity, refinement from deployed anchors, and complete predictions after boundary interventions. They assess Theorem~\ref{thm:mixed-refinement} and Appendix~\ref{app:finite-rf}'s finite-schedule effects.

\begin{table}[!htb]
\centering\small
\caption{Correspondence between theoretical ingredients and supporting evidence. Analytic controls test smooth-system hypotheses; deployed-model diagnostics assess finite-step behavior and complete-sampling effects.}
\label{tab:theory-evidence-map}
\begin{tabularx}{\textwidth}{@{}>{\raggedright\arraybackslash}p{0.25\textwidth}>{\raggedright\arraybackslash}X>{\raggedright\arraybackslash}X@{}}
\toprule
Theoretical ingredient & Proof or direct check & Complementary empirical evidence\\
\midrule
PLAS-map symplecticity and implicit regularity & Symplecticity proof; measured $\mathbf B$ spectra and affine defects & 6,881,280 deployed edge maps; learned-H local approximation order\\
Regular local refinement & Fixed-point and quadratic-remainder proof; nonlinear mixed control & Multi-start refinement and independently checked point Jacobians\\
Smooth-model and correction budgets & Conditional consistency proof; physical-step refinement control & First-order convergence with exact H, vanishing model error, and quadratic corrections\\
Single-refinement deployment & $K=1$ tracking bound; finite-schedule output decomposition & $N=8,12,20$ checks; 512-source, 192-edge refinement comparisons\\
Finite-window feedback and memory & Nonnegative response expansion with a conditional D budget & Physical-coordinate and internal-memory responses; complete RF interventions\\
Regular collision sequences & Finite-event composition bound & 510 contact onsets; matched perturbed targets and event-boundary tails\\
\bottomrule
\end{tabularx}
\end{table}

\subsection{Implicit-map health and smooth consistency controls}
\label{app:theory-health-control}
We measure the implicit block $\mathbf B$ along deployed predictions for 512 sources and two noise realizations over 192 edges. Four successive windows continue from their own predicted endpoints, covering 6,881,280 edge-map evaluations. The minimum singular value is 0.995916 and the maximum condition number is 1.005239. FP64 evaluation of the FP32-generated matrices gives a maximum entrywise symplectic defect of $2.80\times10^{-7}$. Instrumentation leaves the predictions unchanged. These measurements support implicit regularity at the evaluated anchors and the symplectic construction of PLAS maps.

We also test the local affine remainder using the frozen learned H. From 64 sources, two noises, and four positions per predicted window, we obtain 512 states. In FP64, we compare the affine tangent at the anchor-compatible source with an independently iterated implicit step at normalized-state RMS radii $10^{-1},3\times10^{-2},10^{-2},3\times10^{-3},10^{-3}$. The maximum implicit-equation residual is $3.14\times10^{-16}$; fitted error orders range from 1.9492 to 2.0545, with median 2.0020, consistent with Theorem~\ref{thm:affine}'s quadratic remainder.

An independent physical-step control uses the nonlinear nonseparable Hamiltonian
\begin{equation}
 H^\dagger(q,p)=\tfrac12(q^2+p^2)+0.1\sin(q)\sin(p).
\end{equation}
Over duration $T=1.6$, we use 24, 48, 96, 192, and 384 physical steps from three initial states, $(0.7,-0.4)$, $(-1,0.8)$, and $(1.4,1)$. Each implicit step is solved by affine refinement to an anchor-increment tolerance of $10^{-13}$, with a high-accuracy DOP853 solution as the physical reference. Besides exact H, we test gradient error $h(0.1,-0.07)$ and an additive per-step correction $h^2(0.03,-0.02)$. These instantiate the model-error and correction budgets in Corollary~\ref{prop:physical-consistency}.

\begin{table}[htbp]
\centering\small
\caption{Physical-step convergence in the analytic Hamiltonian control. Slopes fit the maximum trajectory error against $h$ over five resolutions; ranges span all three initial states.}
\label{tab:physical-step-control}
\begin{tabular}{@{}lcc@{}}
\toprule
Setting & Budget scaling & Observed order\\
\midrule
Exact Hamiltonian & $\epsilon_H=\eta_h=0$ & 0.9956--1.0060\\
Vanishing gradient error & $\epsilon_H=O(h)$ & 0.9973--1.0065\\
Quadratic step correction & $\eta_h=O(h)$ & 0.9953--1.0066\\
\bottomrule
\end{tabular}
\end{table}
These controls instantiate Corollary~\ref{prop:physical-consistency}'s small-step hypotheses and exhibit first-order consistency. The fixed-step learned-model checks assess implicit residuals, refinement, and complete predictions.

\subsection{Protocol and physical-equation consistency}
All model weights, tree structure, and calibrated coefficients are unchanged. We use Ours with current gradients and full Hessians at every anchor. Fixed-condition refinement freezes the RF state and time, diffusion predictions and tokens, observable attributes, and model parameters, and resets physical-time memory for every complete scan. It repeatedly updates the actual anchor in Eq.~\eqref{eq:refinement-state}. Physical anchors are passed directly between iterations.

The $N=20$ checks cover 512 sources, two noises, and seven first-window RF fields: left/right evaluations at intervals 2, 10, and 18 (zero-based), and the final field. Interval 2 is the first H-bearing interval after diffusion-only initialization. Each uses its deployed anchor and two opposite perturbations of normalized-anchor L2 radius $10^{-4}$. The $N=8,12$ checks use eight sources. Convergence requires anchor and clean sup increments at most $10^{-5}$ for 13 consecutive iterations and clean drift at most $5\times10^{-5}$ over nine iterations, within 64 refinements.

\begin{table}[htbp]
\centering\small
\caption{Fixed-condition checks within 64 refinements. Counts include both noises. \emph{All starts} requires all three initializations to meet the numerical criterion.}
\label{tab:refinement-checks}
\begin{tabular}{@{}rrrrr@{}}
\toprule
$N$ & Sources & Conditions & Warm start & All starts\\
\midrule
8  & 8  & 112 & 108 & 107\\
12 & 8  & 112 & 111 & 111\\
20 & 512 & 7168 & 7030 & 7015\\
\bottomrule
\end{tabular}
\end{table}

\begin{table}[htbp]
\centering\small
\caption{Normalized RMS implicit-equation residuals after the second refinement. This separate panel uses eight sources for $N=8,12$ and 32 for $N=20$, with two noises and seven fields.}
\label{tab:implicit-defects}
\begin{tabular}{@{}rrrrr@{}}
\toprule
$N$ & Median $q$ & Maximum $q$ & Median $p$ & Maximum $p$\\
\midrule
8  & $8.55\times10^{-8}$ & $1.10\times10^{-7}$ & $8.51\times10^{-8}$ & $1.23\times10^{-7}$\\
12 & $8.45\times10^{-8}$ & $1.21\times10^{-5}$ & $8.54\times10^{-8}$ & $4.56\times10^{-6}$\\
20 & $8.41\times10^{-8}$ & $3.06\times10^{-7}$ & $8.02\times10^{-8}$ & $2.09\times10^{-7}$\\
\bottomrule
\end{tabular}
\end{table}

A nonlinear Hamiltonian control includes nonzero residuals, dynamic gates, and recurrent memory. It gives measured local refinement order 1.9799 with fresh Hessians and 0.9933 with a fixed lagged Hessian. The fixed-point discrepancy from an independent implicit mixed rollout is at most $1.11\times10^{-16}$ in both cases. This control tests the refinement mechanism in Theorem~\ref{thm:mixed-refinement}.

Using the norm of the signed gated correction in Eq.~\eqref{eq:physical-correction-consistency} reduces its measured budget by about 4--10\% relative to the sum of component norms. This quantity enters Corollary~\ref{prop:physical-consistency}'s physical-consistency bound under its smoothness and component-consistency assumptions.

\subsection{Actual perturbation stability and numerical boundaries}
A targeted study uses four sources, seven $N=20$ fields, warm and terminal centers, three directions, and normalized-anchor L2 radii $10^{-2},10^{-3},10^{-4}$. Paired gain divides the distance between opposite perturbation outputs by their initial distance in the same 950-dimensional coordinates. Table~\ref{tab:anchor-gains} reports the responses after eight refinements.

\begin{table}[htbp]
\centering\small
\caption{Paired anchor-distance gain after eight refinements, with 84 source--field--direction cases per row.}
\label{tab:anchor-gains}
\begin{tabular}{@{}lrrrr@{}}
\toprule
Center & Radius & Median gain & Maximum gain & Gain $\geq1$\\
\midrule
Warm & $10^{-2}$ & 0.000221 & 0.2204 & 0/84\\
Warm & $10^{-3}$ & 0.002228 & 0.00518 & 0/84\\
Warm & $10^{-4}$ & 0.02429 & 22.0401 & 2/84\\
Terminal & $10^{-2}$ & 0.000225 & 0.2204 & 0/84\\
Terminal & $10^{-3}$ & 0.002186 & 0.00446 & 0/84\\
Terminal & $10^{-4}$ & 0.02293 & 22.0489 & 2/84\\
\bottomrule
\end{tabular}
\end{table}

The expanding cases occur in the final field with different leaves. Replaying all expanding directions gives maximum full-trajectory RMS response $1.42\times10^{-4}$ and positive MSE increase $2.24\times10^{-7}$, quantifying the prediction costs associated with these gain ratios.

In the 512-source panel, batches with an unmet 64-refinement criterion are extended to 256. There remain 137 warm-start and 150 any-start failures among 7,168 conditions; 12 conditions have cross-start anchor or clean sup distances above $10^{-4}$. The largest last-nine clean span is 0.1512. Three sources with larger spans or endpoint discrepancies are additionally checked to 1,024 refinements and retain finite oscillations (last-16 spans $0.00103$--$0.00819$). These cases fall short of the specified refinement criterion; their complete predictions are included in the flagged-source analysis below.

Actual-path tree instrumentation reproduces the deployed leaf selections without changing predictions. Two exact-threshold cells occur in one source, which is also retained in the endpoint analysis. We therefore measure boundary effects through complete sampling, including subsequent leaf selections and state updates.

\paragraph{Point-derivative verification.}
Differentiating the anchor-dependent Hessian requires third derivatives of H. The calculation expresses LayerNorm through explicit centering, variance, and reciprocal square root, and uses FP64 arithmetic. Independent finite differences check these point derivatives.

All 42 directions (14 centers, three directions each) have a resolved scale with unchanged leaves. Across five scales, the best resolved absolute discrepancy per direction has median $2.42\times10^{-11}$ and maximum $1.91\times10^{-9}$; maximum relative discrepancy is $6.61\times10^{-5}$. Four full $950\times950$ Jacobians for one source give:
\begin{center}
\small
\begin{tabular}{@{}lr@{}}
\toprule
Point ($N=20$) & Maximum singular value\\
\midrule
Interval 2, left, warm & 0.0155268\\
Interval 2, left, terminal & 0.0000306129\\
Interval 10, left, warm & 0.000535277\\
Interval 18, left, warm & 0.000759899\\
\bottomrule
\end{tabular}
\end{center}
Independent matrix/JVP checks agree to relative discrepancy $3.25\times10^{-10}$. These are FP64 point derivatives; FP64--FP32 anchor sup distances across 14 batch centers are 0.00315--0.01247, with some leaf changes. The deployed FP32 path is assessed by the complete-sampling tests below.

\subsection{Single-refinement deployment and tail effects}
\label{app:theory-tails}
We test whether extra within-field refinement changes complete predictions beyond the updates already carried across RF intervals. On 512 sources with $N=20$ and two paired noises, all settings use the same accepted-left anchor schedule. Deployment uses $K=1$; numerical references are computed with 8 and 16 within-field refinement updates. First-window MSEs for $K=1,8,16$ are 0.0458448, 0.0458454, and 0.0458455. Table~\ref{tab:theory-endpoints} compares $K=1,16$ over 192 edges, continuing each window from its own predicted endpoint without truth resets.

\begin{table}[htbp]
\centering\small
\caption{192-edge complete predictions with two paired noises. Flagged sources have at least one unmet local criterion at 64 or extended refinement, cross-start disagreement, or exact tree threshold. Positive cost is $100\,\mathbb E[(\ell_1-\ell_{16})_+]/\mathbb E[\ell_{16}]$, taking the positive part before averaging source--noise cases.}
\label{tab:theory-endpoints}
\begin{tabular}{@{}lrrrrr@{}}
\toprule
Sources & Count & $K=1$ MSE & $K=16$ MSE & PhysiFormer & Positive cost (\%)\\
\midrule
All & 512 & 0.250732 & 0.250635 & 0.340289 & 0.1903\\
Flagged & 97 & 0.322510 & 0.321508 & 0.412714 & 0.4744\\
Other & 415 & 0.233955 & 0.234069 & 0.323361 & 0.0990\\
\bottomrule
\end{tabular}
\end{table}

The 192-edge mean $K=1$ minus $K=16$ MSE is $9.76\times10^{-5}$ (source-cluster bootstrap 95\% interval $[-1.55,3.63]\times10^{-4}$). Positive cost is 0.1903\% [0.1120\%, 0.2845\%] overall and 0.4744\% [0.2030\%, 0.8123\%] on flagged sources. The latter group's mean difference is 0.001002 [0.000035, 0.002167]: extra refinement has a small measurable benefit there, while its aggregate prediction advantage over PhysiFormer remains. Bootstrap resampling retains paired noises within sources; grouping by any flagged condition retains both noises.

The largest casewise $K=1$ excess MSE is 0.07138, and some flagged cases have higher error than PhysiFormer. The source with the largest fixed-condition span has deployed MSEs 0.284685 and 0.135125 across the two noises, versus 0.285690 and 0.135129 for $K=16$. For this source, the deployed and additional-refinement losses are closely matched. Overall, warm-started $K=1$ approaches additional-refinement prediction accuracy beyond the updates already carried across RF intervals.

\paragraph{Endpoint sensitivity across RF intervals.}
In a separate 32-source, two-noise panel, we evaluate the 19 complete-interval hybrids in Eq.~\eqref{eq:hybrid-decomposition} for three batches covering typical behavior, warm-start failures, and the largest endpoint discrepancy. In the last case, the $K=1$ versus $K=8$ endpoint RMS distance is 0.039452, bounded by the propagated-increment sum 0.039702; MSE decreases from 0.025498 to 0.023068 (PhysiFormer: 0.014589). Replacing interval 5's refinement initially changes the clean output by only $5.15\times10^{-6}$ RMS without a leaf change, but interval 17's right call changes four leaves and increases the clean difference to 0.006497. Thus, delayed amplification can arise through subsequent discrete transitions.

\paragraph{Anchor interventions through complete sampling.}
On all 512 sources and both noises, we perturb the incoming anchor at one of seven fields by both signs at normalized-anchor L2 radii $10^{-4},10^{-3},10^{-2}$, then recompute all subsequent RF fields. Across 43,008 predictions, positive MSE increases average 0.001674\%, 0.001830\%, and 0.002354\% of the unperturbed 48-edge MSE at the three radii; 38, 42, and 42 cases change leaves. The maximum casewise increase is 0.003636.

All 152 interventions with a leaf change or 48-edge response RMS at least $10^{-4}$ are continued to 192 edges using paired future noise. Unperturbed controls reproduce the original predictions exactly. Positive cost is 0.3245\% of this subset's baseline MSE (source-cluster 95\% interval [0.0380\%, 1.0226\%]); the maximum individual MSE increase is 0.02411. These complete-prediction costs quantify observed boundary tails without assuming contraction across every discrete transition.

\subsection{Diffusion-reset responses with full memory}
\label{app:reset-evidence}
We perturb all 20 physical coordinates individually by $\pm10^{-4}$ and $\pm10^{-3}$ at zero-based edges 0, 16, and 32 in fixed 48-edge fields. The panel contains eight sources across 12 batch fields. D predictions and affine jets are frozen, while residual and memory updates remain active. Each policy is compared to its own unperturbed rollout. Three policies are evaluated: live routing; replay of the unperturbed live gates; and $g=1$ throughout. The 34,560 records comprise 11,520 matched perturbations per policy.

\begin{table}[htbp]
\centering\small
\caption{Finite-field reset-response tails. Sup gain divides maximum coordinate response by the injected coordinate amplitude. Response energy is the mean squared subsequent normalized-state response.}
\label{tab:reset-tails}
\begin{tabular}{@{}lrrr@{}}
\toprule
Routing policy & Sup gain $>10$ (\%) & Maximum sup gain & Response energy\\
\midrule
Live & 0.590 & 1460 & $1.55\times10^{-6}$\\
Fixed live gates & 0.286 & 2604 & $2.61\times10^{-7}$\\
All $H+r$ & 2.352 & 8448 & $3.26\times10^{-5}$\\
\bottomrule
\end{tabular}
\end{table}

Live routing gives a smaller subsequent RMS response than all-$H+r$ in 95.23\% of the matched perturbations (source-cluster 95\% interval [92.11\%, 98.00\%]); the terminal fraction is 94.07\%. The ratio of total response energies is 0.04751, with source-cluster interval [0.00444, 0.11168]. Live cases with sup gain above 10 account for 99.29\% of its perturbation-response energy, whose absolute mean remains $1.55\times10^{-6}$. The response is therefore concentrated in a small tail while aggregate energy is substantially attenuated relative to all-$H+r$.

These fixed-field perturbations preserve the complete memory path and normalize RMS gain by the injected coordinate amplitude. They quantify the finite-window response mechanism in Proposition~\ref{prop:reset}. Complete-sampling comparisons in Appendix~\ref{app:theory-tails} and the internal-memory interventions below measure aggregate end-to-end effects.

\subsection{Internal-memory responses through complete RF sampling}
\label{app:memory-rf-evidence}
We additionally intervene on eight internal memory blocks: the recurrent hidden state, previous gate, previous H candidate, and five accumulated context blocks. The same eight-source panel covers RF intervals 2, 10, and 18, with both left and right calls. At physical edge 16, we perturb one block and recompute the remaining RF process, including subsequent D predictions. Each block uses both signs at amplitude $10^{-3}$, giving 768 source--field--block--sign records with one sampling-noise realization. Hidden-state perturbations use alternating coordinate signs; previous-gate perturbations are clipped to $[0,1]$. Context perturbations scale each entry by $1+|\mathrm{entry}|$, with energy and mass entries kept nonnegative. Zero-intervention replays agree with the ordinary complete prediction within $10^{-5}$ maximum coordinate error.

\begin{table}[htbp]
\centering\small
\caption{Internal-memory perturbations followed by complete RF sampling. Each block has 96 paired cases. Response RMS compares the resulting 48-edge predictions with the unperturbed predictions. MSE changes use a common unperturbed aggregate MSE of 0.0371873; maxima retain individual tails.}
\label{tab:memory-rf}
\begin{tabular}{@{}lrrr@{}}
\toprule
Perturbed block & Maximum RMS & Aggregate MSE change (\%) & Maximum MSE increase\\
\midrule
Recurrent hidden state & 0.000142 & +0.000012 & $1.64\times10^{-7}$\\
Previous gate & 0.000143 & +0.000017 & $3.30\times10^{-7}$\\
Previous H candidate & 0.000143 & +0.000082 & $2.67\times10^{-6}$\\
Contact-provenance context & 0.000142 & +0.000024 & $1.72\times10^{-7}$\\
Cumulative reset & 0.023837 & +0.003641 & $1.26\times10^{-4}$\\
Momentum reset energy & 0.023837 & +0.003445 & $1.26\times10^{-4}$\\
Position reset energy & 0.023906 & +0.059833 & $9.23\times10^{-4}$\\
Reset mass & 0.000142 & +0.000011 & $4.15\times10^{-7}$\\
\bottomrule
\end{tabular}
\end{table}
Across paired memory interventions on eight sources, the largest aggregate MSE increase is approximately 0.060\%, and the maximum individual response distance is 0.02391. These measurements quantify the aggregate prediction cost and response tails, providing empirical support for the finite-window memory-response mechanism described in Proposition~\ref{prop:reset}.

\paragraph{Reproducibility.}
Aggregate prediction losses use FP64 subtraction and accumulation. Model weights remain fixed throughout these experiments. The full-refresh/sparse-refresh comparison is reported in Appendix~\ref{app:lagged-hessian}.

\section{Architecture and Training Details}
\subsection{Architecture and Forward Computation}
\label{app:implementation}
\subsubsection{State conventions and architecture}
For spatial dimension $d$, each active object has $2d$ canonical state coordinates and each position or momentum readout has $d$ outputs; HamiBalls-1 uses $d=2$ and HamiBalls-2 uses $d=3$.

With $n$ active objects, $\mathbf q,\mathbf p\in\R^{nd}$ and the full-system canonical state has dimension $2nd$. All Hamiltonian derivatives in Eq.~\eqref{eq:jet} are taken in physical canonical coordinates; fixed scaling converts candidates and corrections to and from the normalized trajectory representation used by the diffusion model and losses. Under a general normalization, the symplectic form becomes the pulled-back matrix $\mathbf S^\top\mathbf J\mathbf S$, as derived in Appendix~\ref{app:affine-proof}. The fixed normalization statistics define the conversion between physical and normalized coordinates.

\begin{table}[htbp]
\centering\small
\caption{HamiBalls-1 architecture and learned-size accounting. Counts for the router include its leaf-affine terms. Linear coefficients and other learned state are reported separately from neural parameters.}
\label{tab:architecture}
\begin{tabularx}{\linewidth}{@{}lXr@{}}
\toprule
Component & Architecture or state & Count\\
\midrule
Diffusion expert & Width 128; 4 blocks; 4 heads; 2 registers; feed-forward width 256 & 1,159,172\\
Hamiltonian model & Width 16; 2 object-interaction blocks; 4 heads; scalar readout & 4,464\\
Shared residual & SiLU MLP, $154\to32\to32\to4$ & 6,148\\
Router & 12-dimensional projection and physical-time GRU; width-16 $q/p$ offset heads; leaf-affine corrections & 7,588\\
\midrule
Neural subtotal & All four components & 1,177,372\\
Ridge readouts & Fitted linear readouts & 1,056\\
Combined scalars & Neural parameters plus fitted linear coefficients & 1,178,428\\
Other learned state & CART tree (maximum depth four; capacity eight leaves); thresholds; feature statistics & Separate\\
\midrule
PhysiFormer & Diffusion-only backbone; feed-forward width 265 & 1,173,068\\
\bottomrule
\end{tabularx}
\end{table}

The diffusion backbone alternates spatial, temporal, object, and temporal attention blocks. It uses RMS normalization, query/key normalization, SwiGLU feed-forward layers, RF-time modulation, and rotary positional encoding of physical time. There is one state token per object per frame; within-object vertex attention is therefore a singleton operation, whereas across-object attention is nontrivial. The PhysiFormer baseline is trained separately in the same model family with a larger feed-forward width. Our phase-space implementation adapts the trajectory-diffusion architecture of the original PhysiFormer study~\citep{chen2026physiformer}.

\begin{table}[htbp]
\centering\small
\caption{HamiBalls-2 architecture and learned-size accounting. Totals include neural parameters and fitted ridge coefficients; tree structure, thresholds, and normalization statistics are additional state.}
\label{tab:hami2-parameters}
\begin{tabularx}{\linewidth}{@{}lXr@{}}
\toprule
Component & Architecture or state & Count\\
\midrule
Diffusion expert & Width 168; 12 blocks; 8 heads; 2 registers; feed-forward width 336 & 5,615,598\\
Hamiltonian model & Width 32; 2 object-interaction blocks; 4 heads; pair-energy readout & 22,252\\
Shared residual and router & Token-conditioned SiLU residual; physical-time GRU with $q/p$ offset heads and leaf-affine corrections & 32,729\\
\midrule
Neural subtotal & All HamiFormer neural components & 5,670,579\\
Ridge readouts & Fitted linear readouts & 6,960\\
Combined scalars & Neural parameters plus fitted linear coefficients & 5,677,539\\
Other learned state & CART tree (maximum depth four; capacity eight leaves); thresholds; feature statistics & Separate\\
\midrule
Transformer-AR, context 1 & Comparison model & 5,574,318\\
Transformer-AR, context 4 & Comparison model & 5,581,798\\
DHN & Two width-336, two-layer, four-head Transformers & 5,607,170\\
DiT & Comparison model & 5,654,014\\
PhysiFormer & Comparison model & 5,615,598\\
\bottomrule
\end{tabularx}
\end{table}

On HamiBalls-2, DiT uses full spatiotemporal attention and adaLN-Zero under the matched RF protocol. Table~\ref{tab:hami2-parameters} reports the component architecture and exact model sizes. HamiFormer's count includes its Hamiltonian model, shared residual, and router.

HG-DPF on HamiBalls-1 uses a parameter-matched DPF/HNN (785,300/386,386 parameters), trained for 50,000/40,000 AdamW updates with batch size 64. The Perceiver has 64 width-64 latents, two heads, and eight encoder/four decoder self-attention layers; the HNN has four width-248 CELU layers. DPF training predicts noise at 1,000 cosine levels, sampling query length $N_q\in\{5,10,14,19,24,29,34,38,43,48\}$ and context length uniformly in $[1,N_q-1]$. Context comprises independently noised query-prefix copies and the clean initial state. Inference uses 20 DDIM steps, random half-query context, and four-candidate SNIS ($\lambda=1$, stochasticity 1) with standard stochastic-DDIM variance allocation and categorical selection using recomputed clean-estimate energies. Under the 512-source, two-noise, 192-edge protocol, deterministic unguided DPF/SNIS MSEs are 1.307886/1.179656: the larger error is already present without guidance.

DHN retains the reference Hamiltonian update and denoising routines~\citep{deng2025dhn}, with block size and stride one. Observable scene attributes replace the codebook conditioning; padded coordinates are masked, and position/momentum scales are fitted on training data. Training uses the equations-of-motion loss plus $0.1$ times the denoising loss, with the reference coordinatewise corruption. Adam uses learning rate $10^{-4}$ and weight decay $10^{-4}$ for 50,000 updates, each averaging valid-object losses over all 48 transitions in 64 training windows. At inference, each physical edge uses 11 alternating right/left Hamiltonian denoising rounds at levels $0,0.1,\ldots,1$, conditioned on the preceding predicted state. Evaluation uses the same 512 sources, two noise seeds, and 192-edge horizon.

\subsubsection{Observable features and model-tree semantics}
The shared residual reads a diffusion token together with the current noisy state, initial state, mixed predecessor, candidates, and object attributes. A separate feature construction summarizes candidate increments and disagreements, geometry and recovery descriptors, RF level, attributes, and available history. Its inference-time routing history ends at the previous physical edge, and all remaining inputs are current causal observables.

The tree is learned once after the second scalar-router stage, with maximum depth four and capacity for eight leaves. CART splits on observable features to reduce normalized squared-error objectives for candidate mixing and Hamiltonian correction in both position and momentum. For feature vector $\vphi\in\R^m$, the design vector is
\begin{equation}
 \vpsi(\vphi)=\left[\clip\left((\vphi-\vmu_{\vphi})\oslash\mathbf s_{\vphi},-8,8\right);1\right]\in\R^{m+1},
\end{equation}
where division is componentwise and the positive normalization scales are fixed during deployment. Each leaf has separate position and momentum readouts with output dimension $d$. After fitting the leaf readouts, $\alpha_q,\alpha_p\in[0,1]$ are fitted separately by closed-form constrained least squares against the required Hamiltonian corrections, using held-out sources within the training split, and fixed during inference.
On the first physical edge, setting $\Delta\mathbf r_\ell=\mathbf 0$ yields the shared contribution $(\mathbf 1-\valpha)\odot\mathbf r_0$. The leaf correction to the router is affine in standardized current observables and an intercept. For numerical evaluation of a base probability's logit, probabilities are clamped to $[10^{-6},1-10^{-6}]$; this clamp serves numerical stability.

Tree structure, feature standardization, and component coefficients are fitted from the training split and fixed during inference.

\subsubsection{One complete mixed clean estimate}
Algorithm~\ref{alg:clean} specifies the causal ordering within a field evaluation. Anchors are mixed-coordinate pairs $(\bar{\mathbf q}_{k-1},\bar{\mathbf P}_k)$, with $\bar{\mathbf q}$ from the committed mixed predecessor and $\bar{\mathbf P}$ from its uncorrected H candidate. The first position is the observed $\mathbf q_0$. The derivative construction is batched over physical edges. The subsequent scan is sequential because both the residual/router and the next Hamiltonian candidate depend on the newly mixed state.

\begin{algorithm}[tbp]
\caption{Mixed clean estimate with one PLAS refinement ($K=1$)}
\label{alg:clean}
\begin{algorithmic}[1]
\Require Noisy future block $\mathbf X_\tau$, RF time $\tau$, initial state $\mathbf z_0$, attributes $\mathbf c$, frozen mixed-coordinate anchor $\mathbf a$, diffusion-only indicator $b_D$
\State Evaluate $D(\mathbf X_\tau,\tau;\mathbf z_0,\mathbf c)$; retain its clean trajectory and tokens.
\If{$b_D$}
 \State Form $\mathbf a^D$ from D predecessor positions and D target momenta, keeping $\mathbf q_0$ fixed.
 \State \Return the normalized diffusion clean trajectory and $\mathbf a^D$.
\EndIf
\State Construct all $(\mathbf A_k,\mathbf b_k)_{k=1}^L$ from the fixed anchors, using Eqs.~\eqref{eq:jet} and~\eqref{eq:affine}.
\State Set $\widehat{\mathbf z}_0=\mathbf z_0$ and reset physical-time router history.
\For{$k=1,\ldots,L$}
 \State Form $\widehat{\mathbf z}^H_k=\mathbf A_k\widehat{\mathbf z}_{k-1}+\mathbf b_k$ and obtain $\widehat{\mathbf z}^D_k$.
 \State Evaluate the shared residual and observable feature construction.
 \State Route to $\ell=\mathcal T(\vphi)$ and evaluate $\Delta\mathbf r_\ell=\mathbf W_\ell\vpsi$.
 \If{$k=1$} \State Set $\Delta\mathbf r_\ell=\mathbf 0$. \EndIf
 \State Form $\mathbf r_k$ using Eq.~\eqref{eq:residual}, then form corrected $\widehat{\mathbf z}^{\HR}_k$ as in Eq.~\eqref{eq:mix}.
 \State Evaluate the recurrent router on the corrected candidate and compute $\mathbf g_k$.
 \State Mix with Eq.~\eqref{eq:mix}; update physical-time history using this edge.
\EndFor
\State Form $\mathbf a^+\gets((\widehat{\mathbf q}_{k-1})_{k=2}^L,(\widehat{\mathbf p}^H_k)_{k=1}^L)$ in physical coordinates.
\State \Return the full normalized mixed trajectory and $\mathbf a^+$.
\end{algorithmic}
\end{algorithm}

\subsubsection{Stateful RF sampler}
\label{app:sampling}
The sampler carries the noisy RF trajectory and a separate mixed-coordinate anchor. Both evaluations of a Heun interval receive the same incoming anchor, which remains frozen throughout the interval. After numerical acceptance, the anchor returned by the left field evaluation is committed. Algorithm~\ref{alg:sampler} specifies this ordering.

\begin{algorithm}[tbp]
\caption{RF sampling with accepted-left mixed-coordinate anchor commits}
\label{alg:sampler}
\begin{algorithmic}[1]
\Require Initial state $\mathbf z_0$, attributes $\mathbf c$, intervals $N$, denominator floor $\varepsilon_\tau=0.05$
\State Draw $\mathbf X_0\sim\mathcal N(\mathbf 0,\Id)$; set $\Delta=1/N$ and initialize the anchor from repeated $\mathbf z_0$.
\For{$s=0,\ldots,N-1$}
 \State $\tau_s=s/N$; $b_D\gets(s<2)$.
 \State $(\mathbf Y^L,\mathbf a^L)\gets\operatorname{Clean}(\mathbf X_s,\tau_s,\mathbf z_0,\mathbf c,\mathbf a,b_D)$ using Algorithm~\ref{alg:clean}.
 \State $\mathbf v^L\gets(\mathbf Y^L-\mathbf X_s)/\max(1-\tau_s,\varepsilon_\tau)$.
 \If{$s=N-1$}
  \State $\mathbf X_{s+1}\gets\mathbf X_s+\Delta\mathbf v^L$.
 \Else
  \State $\mathbf P_s\gets\mathbf X_s+\Delta\mathbf v^L$.
  \State $(\mathbf Y^R,\mathbf a^R)\gets\operatorname{Clean}(\mathbf P_s,\tau_{s+1},\mathbf z_0,\mathbf c,\mathbf a,b_D)$.
  \State $\mathbf v^R\gets(\mathbf Y^R-\mathbf P_s)/\max(1-\tau_{s+1},\varepsilon_\tau)$.
  \State $\mathbf X_{s+1}\gets\mathbf X_s+\tfrac\Delta2(\mathbf v^L+\mathbf v^R)$.
 \EndIf
 \State $\mathbf a\gets\mathbf a^L$ after accepting the interval; keep $\mathbf q_0$ fixed.
\EndFor
\State \Return $\mathbf X_N$.
\end{algorithmic}
\end{algorithm}

The first two \emph{complete intervals} use the diffusion expert for both Heun evaluations. The hybrid part consequently begins at $\tau=2/N$, namely $0.25$, $1/6$, or $0.10$ for $N=8,12,20$. The corresponding numbers of field evaluations are $15,23,39$. Each hybrid field evaluation produces an entire $L$-edge physical rollout; $N$ counts RF sampling intervals.

The deployed sampler checks $1/N\geq0.05$. Therefore the last Euler update has $d_{\tau_{N-1}}=\Delta$, and $\mathbf X_N=\mathbf Y^L$ on that interval. Appendix~\ref{app:finite-rf} gives finite-schedule perturbation accounting with this denominator floor.

\subsection{Training Objectives and Schedule}
\label{app:training}
The staged objectives and carrier schedule below apply to both HamiBalls datasets.
\subsubsection{Joint base-model training}
The joint base stage minimizes $\mathcal L_{\rm base}=\mathcal L_D+\mathcal L_H$ with the settings in Table~\ref{tab:training-config}. RF time follows a logit-normal distribution with $\mu_\tau=-0.8$, $\sigma_\tau=0.8$. Both base networks remain frozen throughout downstream training.

\begin{table}[htbp]
\centering\scriptsize
\caption{Training settings for HamiBalls-1 (H1) and HamiBalls-2 (H2).}
\label{tab:training-config}
\begin{tabularx}{\linewidth}{@{}l>{\centering\arraybackslash}X>{\centering\arraybackslash}X@{}}
\toprule
Setting & H1 & H2\\
\midrule
Base updates; batch size & 50,000; 64 & 50,000; 64\\
Optimizer; weight decay; grad clip & AdamW; $10^{-4}$; 1.0 & AdamW; $10^{-4}$; 1.0\\
Peak learning rate: D / H & $3\times10^{-4}$ / $10^{-4}$ & $3\times10^{-4}$ / $10^{-4}$\\
Peak LR: residual / scalar / final q/p & $3\times10^{-4}$ / $3\times10^{-4}$ / $3\times10^{-3}$ & $10^{-3}$ / $3\times10^{-3}$ / $3\times10^{-3}$\\
RF denominator floor & 0.05 & 0.05\\
Downstream stages & $4\times200$ updates & $4\times200$ updates\\
Local replay length; cache refresh & 12 edges; 8 updates & 12 edges; 8 updates\\
Tree capacity & depth 4; 8 leaves & depth 4; 8 leaves\\
Residual weights (local; recovery/hull) & $1/4;1/8$ & $1/4;1/8$\\
No-harm weight; relative ridge & $0.25;0.01$ & $0.25;0.01$\\
\bottomrule
\end{tabularx}
\par\smallskip
\begin{minipage}{\linewidth}
Learning-rate schedules use linear warmup and cosine decay where applicable.
\end{minipage}
\end{table}

The Hamiltonian contribution $\mathcal L_H$ uses local adjacent-state relations. For an observed adjacent pair $(\mathbf z_k,\mathbf z_{k+1})$, let $\mathbf m_k=(\mathbf z_k+\mathbf z_{k+1})/2$, let $\mathbf s_z$ denote the fixed state scale, and define the normalized relation residual
\begin{equation}
 \mathbf e_k=\left(h\mathbf J\nabla H_\theta(\mathbf m_k)-(\mathbf z_{k+1}-\mathbf z_k)\right)\oslash\mathbf s_z.
\end{equation}
For each canonical coordinate, a fixed anisotropic scale $a_j$ is the median absolute normalized adjacent-state increment over the training sources, edges, and objects. Degenerate scales are rejected. With $\langle\cdot\rangle$ denoting an average over coordinates, the robust weight for object $i$ at edge $k$ and loss are
\begin{equation}
 u_{k,i}=\left\langle(\mathbf e_{k,i}\oslash\mathbf a)^2\right\rangle,\quad
 w_{k,i}=\frac{\nu}{\nu+u_{k,i}},\quad \nu=4,\quad
 \mathcal L_H=\frac{\nu+1}{\nu}\E_{k,i}\left[
 \operatorname{sg}(w_{k,i})\frac12\langle\mathbf e_{k,i}^2\rangle\right].
 \label{eq:h-robust}
\end{equation}
Here $\operatorname{sg}$ denotes stop-gradient. The anisotropic scale determines retention weights, while the quadratic term uses the original normalized residual. This iteratively reweighted objective is evaluated on sampled observable adjacent-state pairs.

\subsubsection{Shared-residual learning}
The residual is freshly initialized and trained for 200 updates with fixed $H/D$, using 12-edge local replay within cached 48-edge random-gate trajectories refreshed every eight updates. For $c\in\{q,p\}$, let $E_c,I_c$ be the coordinate-mean normalized squared correction error and ideal-correction energy at each source--edge--object entry on the local teacher-forced branch. Its local term is
\begin{equation}
 \mathcal L_{{\rm local},c}
 =\E_{b,k,i}\!\left[\frac{\log(1+E_c)}{1+I_c}+\bigl[\sqrt{E_c}-\sqrt{I_c}\bigr]_+\right],
 \label{eq:local-r}
\end{equation}
where $[a]_+=\max(a,0)$ and $\E_{b,k,i}$ averages valid source--edge--object entries. For closed-loop terms, use corrected candidate $\mathbf C_c$, diffusion candidate $\mathbf D_c$, target $\mathbf z_c^*$, and fixed scales $\mathbf s_c>\mathbf0$:
\begin{equation}
 \mathbf v_c=(\mathbf C_c-\mathbf D_c)\oslash\mathbf s_c,\quad
 \mathbf u_c=(\mathbf z_c^*-\mathbf D_c)\oslash\mathbf s_c,\quad
 \ell_{\rm PH}(\mathbf u)=\sqrt{1+\langle\mathbf u^2\rangle}-1.
 \label{eq:residual-scaled-errors}
\end{equation}
\begin{equation}
 \theta_c=\operatorname{sg}\!\left[\clip\!\left(
 \frac{\mathbf u_c^\top\mathbf v_c}{\max(\|\mathbf v_c\|_2^2,\epsilon_{\rm hull})},0,1\right)\right].
 \label{eq:hull-projection}
\end{equation}
\begin{equation}
 \mathcal L_{{\rm recovery},c}=\E_{b,k,i}[\ell_{\rm PH}(\mathbf v_c-\mathbf u_c)],\qquad
 \mathcal L_{{\rm hull},c}=\E_{b,k,i}[\ell_{\rm PH}(\theta_c\mathbf v_c-\mathbf u_c)].
 \label{eq:recovery-hull}
\end{equation}
Here $\epsilon_{\rm hull}>0$ is a numerical floor. The target, diffusion candidate, and projection coefficient are detached; residual learning differentiates through $\mathbf C_c$. Recovery penalizes corrected-candidate error, while hull penalizes the target's distance to the candidate segment. The total objective is
\begin{equation}
 \mathcal L_r=\frac14\left(\mathcal L_{{\rm local},q}+\mathcal L_{{\rm local},p}\right)
 +\frac18\sum_{c\in\{q,p\}}\left(\mathcal L_{{\rm recovery},c}+\mathcal L_{{\rm hull},c}\right).
 \label{eq:r-full}
\end{equation}

\subsubsection{Router losses: endpoint preference, recurrent regret, and risk}
\label{app:gate-loss}
A training cell is indexed by source, physical edge, and object; $c\in\{q,p\}$ denotes its $d$-coordinate component. Let $\mathbf D_c,\mathbf C_c$ denote its diffusion and residual-corrected Hamiltonian candidates, and let $\mathbf z_c^*$ be the target. For a gate $g$, write $\mathbf M_c(g)=\mathbf D_c+g(\mathbf C_c-\mathbf D_c)$. All targets, labels, normalizing headroom, and oracle coefficients described below are detached when used as supervision.

\paragraph{Endpoint preference.}
Let $e_D=\MSE(\mathbf D_c,\mathbf z_c^*)$, $e_C=\MSE(\mathbf C_c,\mathbf z_c^*)$ in normalized state coordinates. Define the binary preferred endpoint $y=\mathbf1[e_C<e_D]$ and the endpoint regret $a=|e_D-e_C|$. Exact ties receive zero weight. When both endpoint classes have positive regret mass, each class receives total weight one half, distributed proportionally to $a$ within that class. A single active class receives total weight one. For normalized weights $\omega_j$ and router logits $\lambda_j$, the loss is
\begin{equation}
 \mathcal L_{{\rm end},c}=\sum_j\omega_j\left[\log(1+\exp(\lambda_j))-y_j\lambda_j\right].
 \label{eq:gate-end}
\end{equation}
An all-tie batch contributes zero. The label encodes which candidate endpoint has lower MSE.

\paragraph{Global recurrent projection regret.}
For the final dual-channel recurrent term, position error is evaluated in physical position units and momentum error in one-step position-equivalent units, $h\,\Delta\mathbf p/m_i$. In either component, let $\mathbf a_j$ be the diffusion error and $\mathbf b_j$ the candidate disagreement in these units. Define
\begin{equation}
 E_j(g)=\langle(\mathbf a_j+g\mathbf b_j)^2\rangle,\qquad
 \theta_j=\clip\left(-\frac{\langle\mathbf a_j\mathbf b_j\rangle}{\langle\mathbf b_j^2\rangle},0,1\right),
\end{equation}
with any minimizing value, chosen as zero, when the disagreement vanishes. The pooled recurrent regret is
\begin{equation}
 \mathcal L_{{\rm rec},c}
 =\frac{\sum_j[E_j(g_j)-E_j(\theta_j)]_+}
 {\sum_j[E_j(1/2)-E_j(\theta_j)]_++\epsilon_{\rm rec}}.
 \label{eq:gate-rec}
\end{equation}
Normalization uses pooled midpoint headroom with a stabilizer $\epsilon_{\rm rec}>0$. Physical-time router recurrence is retained when computing $g_j$. The first scalar-router stage uses the analogous global sequence projection regret for its single shared gate; the second adds the componentwise penalty below. Cached carrier trajectories are treated as fixed data within each optimization block.

\paragraph{Componentwise non-degradation penalty.}
For the componentwise penalty, errors $e_{bki}(g)$ are MSEs in normalized state coordinates, for source $b$, edge $k$, and object $i$. With a bar denoting the mean over edges,
\begin{equation}
 \mathcal L_{{\rm nh},c}=\operatorname{mean}_{b,i}
 \frac{\overline{[e_{bki}(g)-e_{bki}(0)]_+}}
 {\operatorname{sg}\!\left(\overline{[e_{bki}(1)-e_{bki}(0)]_+}\right)}.
 \label{eq:gate-nh}
\end{equation}
A ratio is defined to be zero when its denominator is at most machine epsilon; active denominators are numerically clamped. This empirical penalty averages positive degradation over physical edges. The second scalar-router stage and the final stage use weight $0.25$ for the corresponding component-averaged penalty.

\paragraph{Momentum tail term.}
Using the position-equivalent errors $E$ from the recurrent term, first average the error \emph{difference} over physical edges and then take its positive part:
\begin{equation}
 \begin{aligned}
 h_{bi}&=\left[\overline{E_{bki}(g)-E_{bki}(0)}\right]_+,\qquad
 h^H_{bi}=\left[\overline{E_{bki}(1)-E_{bki}(0)}\right]_+,\\
 \mathcal L_{{\rm tail},p}&=\frac{\operatorname{RMS}_{b,i}(h)}
 {\operatorname{sg}(\operatorname{RMS}_{b,i}(h^H))}.
 \end{aligned}
 \label{eq:gate-tail}
\end{equation}
The same zero-denominator convention applies, and the zero vector uses the canonical zero subgradient. Equation~\eqref{eq:gate-tail} averages over time before taking the positive part, while Eq.~\eqref{eq:gate-nh} takes the positive part first. The tail term preserves the absolute severity of source--object harm relative to a batch-level scale.

\paragraph{Final objective.}
The two component objectives and their joint loss are
\begin{align}
 \mathcal L_q&=\tfrac12(\mathcal L_{{\rm end},q}+\mathcal L_{{\rm rec},q})
                  +\tfrac14\mathcal L_{{\rm nh},q},\\
 \mathcal L_p&=\tfrac13(\mathcal L_{{\rm end},p}+\mathcal L_{{\rm rec},p}+\mathcal L_{{\rm tail},p})
                  +\tfrac14\mathcal L_{{\rm nh},p},\\
 \mathcal L_g&=\tfrac12(\mathcal L_q+\mathcal L_p).
 \label{eq:gate-final}
\end{align}
The tail term is applied to momentum. These soft penalties target candidate preference and non-degradation on the training carriers. Proposition~\ref{prop:reset} describes the finite-window diffusion and propagated error budgets.

\subsubsection{Conditional ridge fitting and training partitions}
For a fixed carrier and tree, let $\mathbf y_j$ be the normalized target correction from the pure Hamiltonian candidate and $\vpsi_j$ the standardized feature vector with intercept. Write the hierarchical fit as a common matrix $\mathbf W_0$ plus leaf deviation $\mathbf V_\ell$:
\begin{equation}
 \min_{\mathbf W_0,\{\mathbf V_\ell\}}
 \sum_j\omega_j\norm{\mathbf y_j-(\mathbf W_0+\mathbf V_{\ell(j)})\vpsi_j}^2
 +\lambda\left(\norm{\mathbf W_0\mathbf P^{1/2}}_F^2+
                    \sum_\ell\norm{\mathbf V_\ell\mathbf P^{1/2}}_F^2\right),
 \label{eq:ridge}
\end{equation}
Here $\omega_j$ equalizes total weight across physical sources and is constant for equal row counts; $\mathbf P$ is diagonal with unit feature penalties and intercept penalty $0.01$. The relative ridge strength is $0.01$ times the mean diagonal of the common-block design Gram matrix. Fitting is componentwise on collected carriers; deployment uses the merged leaf matrices $\mathbf W_\ell=\mathbf W_0+\mathbf V_\ell$ and the calibrated blend in Eq.~\eqref{eq:residual}. The final stage accumulates design sufficient statistics across refreshes.

\begin{table}[htbp]
\centering\small
\caption{Training sequence. The diffusion and Hamiltonian losses form one 50,000-update base stage. A refresh regenerates rollout carriers. Fixed previous-stage policies generate the scalar-stage pools.}
\label{tab:training}
\begin{tabularx}{\linewidth}{@{}l r X@{}}
\toprule
Trainable component & Updates & Objective and carrier policy\\
\midrule
Diffusion $+$ Hamiltonian base & 50,000 & $\mathcal L_D+\mathcal L_H$; whole-window RF (batch 64) and local relation minibatches\\
Shared residual & 200 & Eq.~\eqref{eq:r-full}; random-gate carriers; refresh every 8 updates\\
First scalar router & 200 & Global sequence projection regret; fixed mixed-$N$ pool\\
Second scalar router & 200 & Projection regret plus non-degradation penalty; fixed pool from the preceding scalar router\\
Model tree & --- & Fit once after scalar stages outside gradient-based optimization\\
Dual-channel/leaf router & 200 & Eq.~\eqref{eq:gate-final}; 25 refreshes of 8 updates; cumulative ridge fitting\\
\bottomrule
\end{tabularx}
\end{table}

The final router stage trains the added $q/p$ heads and leaf-affine routing terms, with the scalar trunk frozen. Each final refresh contains 64 physical sources, split into two \emph{disjoint} 32-source groups with independent noise streams. Each group contains 16 sources sampled with $N=8$, 10 with $N=12$, and 6 with $N=20$. The two scalar stages train on fixed pools generated by their preceding policies, with source, noise, and random-routing assignments fixed within each pool.

Tree fitting, calibration, and ridge solves add computation outside the optimizer-update count in Table~\ref{tab:training}.

\section{Numerical Methods and Computational Efficiency}
\label{app:numerical-details}
This appendix records numerical consequences for the implemented sampler and the optional low-frequency Hessian refresh. These analyses use the same distinction between anchor transition $\mathcal R_{\mathbf y}$ and clean readout $\mathcal C_{\mathbf y}$ as Eq.~\eqref{eq:refinement-state}.

\subsection{Integrator Variants and Evaluation Protocol}
\label{app:integrator-ablation}
The eight sampled zero-based physical edges are $\{0,7,13,20,27,34,40,47\}$. Each solver is evaluated at its own realized mixed predecessors on the same 512 sources and two noise realizations. The reference is an independently converged FP64 Newton solve, requiring normalized residual and increment below $10^{-12}$ and a change below $10^{-12}$ after two additional iterations. All reference checks pass. SymEuler2 performs two momentum fixed-point updates followed by a final position-gradient evaluation. For each component, Table~\ref{tab:integrators} reports the square root of the mean squared error over the sampled entries and accepted RF fields. The left and center panels of Figure~\ref{fig:integrator-ablation} report the arithmetic cumulative mean of the separately aggregated per-field RMSEs. Their final ordinates therefore have a different aggregation from the table.

Timings are medians of 11 warmed, synchronized B64$\times$48 sampling calls on an A800 80GB PCIe GPU with $N=20$. Model loading, data I/O, compilation, and graph capture are excluded. Each solver computes its candidate before the same compiled residual/router stage; applicable derivative compilation, kernel autotuning, and attention contractions are shared across methods. The comparison uses FP32 model evaluation. PLAS-DC and PLAS-DC-S accumulate the defect correction in FP64; PLAS-DC-F uses FP32 correction arithmetic. End-to-end Ours timings using the deployed scan are reported separately in Appendix~\ref{app:inference-wallclock}.

\paragraph{Unified defect correction.}
Work in physical coordinates and suppress the physical-edge index. Let
$\widetilde{\mathbf A}\mathbf z+\widetilde{\mathbf b}$ be the selected affine jet from Eq.~\eqref{eq:affine}, with blocks $\widetilde{\mathbf B},\widetilde{\mathbf V}$. These blocks may be current or cached as specified below; assume $\widetilde{\mathbf B}$ is nonsingular. For the actual mixed predecessor $\mathbf z=[\mathbf q^\top,\mathbf p^\top]^\top$, first obtain the affine momentum seed and its implicit-equation defect:
\begin{equation}
 \widehat{\mathbf P}=
 \begin{bmatrix}\mathbf 0&\Id\end{bmatrix}
 (\widetilde{\mathbf A}\mathbf z+\widetilde{\mathbf b}),\qquad
 \mathbf d=\mathbf p-\widehat{\mathbf P}-hH_{\mathbf q}(\mathbf q,\widehat{\mathbf P}).
 \label{eq:dc-seed}
\end{equation}
The full tangent correction is
\begin{equation}
 \widetilde{\mathbf B}\,\boldsymbol\delta=\mathbf d,\qquad
 \begin{bmatrix}\mathbf Q_{\rm DC}\\\mathbf P_{\rm DC}\end{bmatrix}
 =\begin{bmatrix}\mathbf q+hH_{\mathbf p}(\mathbf q,\widehat{\mathbf P})\\\widehat{\mathbf P}\end{bmatrix}
 +\begin{bmatrix}\widetilde{\mathbf V}\\\Id\end{bmatrix}\boldsymbol\delta.
 \label{eq:dc-update}
\end{equation}
Both Hamiltonian gradients are freshly evaluated at $(\mathbf q,\widehat{\mathbf P})$. The momentum update is a Newton-type correction using the selected approximation to $\Id+hH_{\mathbf q\mathbf p}(\mathbf q,\widehat{\mathbf P})$; the position update transports the same increment through the selected approximation to $hH_{\mathbf p\mathbf p}$. Thus both components are corrected by one coupled update. In the implementation, the momentum block column of $\widetilde{\mathbf A}$ has upper and lower blocks $\widetilde{\mathbf V}\widetilde{\mathbf B}^{-1}$ and $\widetilde{\mathbf B}^{-1}$, respectively, so its multiplication by $\mathbf d$ supplies both increments without another Hessian evaluation. The resulting candidate is normalized and passed once to the residual/router stage before advancing the mixed scan.

\paragraph{Refresh policies.}
In the method names, DC denotes \emph{defect correction}, F denotes a \emph{frozen affine jet}, and S denotes \emph{sparse Hessian refresh}. The frozen-jet variant holds both the affine matrix and offset fixed within a sampling chunk; the sparse-Hessian variants recompute anchor gradients and offsets while refreshing the Hessian intermittently.

Each physical edge has its own affine jet. PLAS constructs these jets using current anchor gradients and complete Hessians and applies the affine map directly. PLAS-DC uses the same current jets and applies Eqs.~\eqref{eq:dc-seed}--\eqref{eq:dc-update}. PLAS-DC-F constructs the window's jets at the first Hamiltonian-bearing anchor and holds their matrices and offsets fixed throughout that sampling chunk; the gradients in Eqs.~\eqref{eq:dc-seed}--\eqref{eq:dc-update} remain fresh at every physical step.

PLAS-S and PLAS-DC-S use the lagged-Hessian construction in Appendix~\ref{app:lagged-hessian}: they recompute anchor gradients and affine offsets at every distinct anchor, while refreshing the complete Hessian at anchor ordinals $0,16,32,\ldots$. Exact reuse of an identical anchor does not advance this counter. The Hessian cache and counter reset at the start of each sampling chunk. PLAS-S applies the resulting affine map directly; PLAS-DC-S additionally applies Eqs.~\eqref{eq:dc-seed}--\eqref{eq:dc-update}, including fresh physical-step gradients. All variants retain the same RF schedule, weights, residuals, and routing rules.

The underlying affine maps retain the PLAS symplectic identity. Defect correction augments these maps through Eqs.~\eqref{eq:dc-seed}--\eqref{eq:dc-update}. The RMSE values and Pareto boundary report measured point estimates under the arithmetic settings above.

\paragraph{End-to-end prediction under integrator substitutions.}
\label{app:integrator-endpoints}
Table~\ref{tab:integrator-endpoints} evaluates 512 HamiBalls-1 sources with two paired noises over four 48-edge windows, initialized at episode start and continued from predicted endpoints. PLAS achieves prediction MSE close to SymEuler2 at lower runtime than either sequential solver. Against Explicit Euler, its 192-edge MSE is 1.29\% lower, with 10.5\% lower sampling time; its sampling time is 37.7\% lower than SymEuler2's.
\begingroup
\setlength{\intextsep}{6pt plus 1pt minus 1pt}
\begin{table}[!htbp]
\centering\small
\caption{HamiBalls-1 prediction accuracy and sampling time ($\downarrow$). MSE uses 512 sources and two noises; wall-clock is per B64$\times$48 sampling call under the shared timing protocol of Table~\ref{tab:integrators}.}
\label{tab:integrator-endpoints}
\begin{tabularx}{\textwidth}{l*{3}{>{\centering\arraybackslash}X}}
\toprule
Integrator & 48-edge $z$-MSE & 192-edge $z$-MSE & Wall-clock (s)\\
\midrule
Explicit Euler & 0.045910 & 0.253974 & 0.8875\\
SymEuler2 & 0.045846 & 0.250647 & 1.2752\\
PLAS & 0.045850 & 0.250707 & \textbf{0.7940}\\
\bottomrule
\end{tabularx}
\end{table}
\endgroup

\subsection{Finite-refinement deployment error}
\label{app:finite-rf}
\paragraph{Anchor and field errors.}
For a field call that performs $K$ refinements, its output is
$\mathcal C_{\mathbf y}^{[K]}(\mathbf a)=\mathcal C_{\mathbf y}(\mathcal R_{\mathbf y}^{K-1}(\mathbf a))$.
A converged reference uses $\mathcal C_{\mathbf y}(\mathbf a^*(\mathbf y))$.
Let $e_0=\norm{\mathbf a-\mathbf a^*}\leq r$ in an invariant regular neighborhood with quadratic constant $C>0$ and $Cr<1$. Induction on Theorem~\ref{thm:mixed-refinement} gives the finite-$K$ bound
\[
 \norm{\mathcal R_{\mathbf y}^{K}(\mathbf a)-\mathbf a^*}
 +\norm{\mathcal C_{\mathbf y}^{[K]}(\mathbf a)-\mathcal C_{\mathbf y}(\mathbf a^*)}
 \leq C^{-1}(Ce_0)^{2^K}.
\]
The induction uses $e_{j+1}\leq Ce_j^2$ and the last-update readout bound. For $K=1$, $e_0\leq e_{s-1}+d_s$ (Appendix~\ref{app:tracking-proof}). Apply the fixed normalization and divide by $d_\tau\geq\varepsilon_\tau$ to bound field error, comparing each call at its own RF state.

For a more general local tracking estimate
$e_s\leq\gamma_s e_{s-1}+\gamma_s d_s+b_s$,
where $d_s$ bounds target motion and $b_s$ includes implementation error or a bounded branch jump, repeated substitution gives
\begin{equation}
 e_s\leq
 \left(\prod_{i=1}^{s}\gamma_i\right)e_0+
 \sum_{j=1}^{s}\left(\prod_{i=j+1}^{s}\gamma_i\right)
                       (\gamma_jd_j+b_j).
\end{equation}
The gain and jump bounds apply along the compared states; suffix products quantify downstream amplification. Within regular regions, $b_s$ can account for finite precision; at hard-tree crossings, $b_s$ additionally accounts for the bounded discrepancy between branches.

\paragraph{Conditional Heun perturbation bound.}
Let $\mathbf v^*$ be the converged-refinement field, $L_v$-Lipschitz in state on a region containing both sampling trajectories and their predictors. Let $\xi_s^L,\xi_s^R$ bound approximate-field discrepancies from $\mathbf v^*$ at the actual left state and actual right predictor. Comparing the actual predictor to the exact-field predictor from the same left state gives a distance at most $\Delta_s\xi_s^L$. Therefore the one-interval update discrepancy is at most
\begin{equation}
 \Delta_s w_s,\qquad
 w_s=\tfrac12\bigl[(1+L_v\Delta_s)\xi_s^L+\xi_s^R\bigr]
 \label{eq:heun-local-perturbation}
\end{equation}
for Heun; on the final Euler interval $w_s=\xi_s^L$. The exact-field Heun update has state Lipschitz bound
$1+L_v\Delta_s+\tfrac12L_v^2\Delta_s^2$. With equal initial states and
$\sum_s\Delta_s=1$, unrolling the interval recurrence yields
\begin{equation}
 \norm{\mathbf X_N-\widetilde{\mathbf X}_N}
 \leq \exp\!\left(L_v+\tfrac12L_v^2\Delta_{\max}\right)
                       \sum_s\Delta_s w_s.
 \label{eq:rf-perturbation}
\end{equation}
This finite-schedule estimate applies on regions with the stated field regularity. Discontinuities at hard-tree crossings enter through explicit jump terms; the full-sampling comparisons measure their downstream effects.

\paragraph{Exact paired-sampling decomposition.}
Let $F_j,G_j$ denote the $S=N-2$ hybrid RF-interval transitions for $K=1$ and a finite-$K$ comparator on extended state (RF trajectory, committed anchor, and schedule metadata), including both Heun calls and the final readout. From the common state $\mathbf x_0$ after two diffusion-only intervals, let $\mathbf P$ extract the final trajectory and define
\begin{equation}
 \mathbf h_j=\mathbf P\,G_S\cdots G_{j+1}F_j\cdots F_1(\mathbf x_0),\qquad
 \Delta\mathbf h_j=\mathbf h_j-\mathbf h_{j-1},\qquad
 B_{\rm end}=\sum_{j=1}^{S}\norm{\Delta\mathbf h_j}_{\rm RMS}.
 \label{eq:hybrid-decomposition}
\end{equation}
Cancellation gives $\mathbf h_S-\mathbf h_0=\sum_j\Delta\mathbf h_j$ and hence
$\norm{\mathbf h_S-\mathbf h_0}_{\rm RMS}\leq B_{\rm end}$. With
$\mathcal L(\mathbf h)=\norm{\mathbf h-\mathbf Z^\dagger}_{\rm RMS}^2$,
\begin{equation}
 \mathcal L(\mathbf h_S)-\mathcal L(\mathbf h_0)
 =\sum_j\bigl[\mathcal L(\mathbf h_j)-\mathcal L(\mathbf h_{j-1})\bigr],\qquad
 |\mathcal L(\mathbf h_S)-\mathcal L(\mathbf h_0)|
 \leq 2\sqrt{\mathcal L(\mathbf h_0)}B_{\rm end}+B_{\rm end}^2.
 \label{eq:hybrid-loss-budget}
\end{equation}
These identities apply to finite-refinement sampling paths, including discrete tree transitions: every downstream gate, leaf, memory update, and diffusion evaluation is recomputed. They measure endpoint discrepancies, separately from the local convergence guarantee. The comparator uses $K=8$; the 57 complete hybrid evaluations in Appendix~\ref{app:theory-tails} use this procedure.

\subsection{Low-frequency Hessian refresh}
\label{app:lagged-hessian}
This optional implementation retains current Hamiltonian gradients while reusing a symmetric Hessian from an earlier RF field. The approximation is temporal: all Hessian blocks are retained at a refresh.

\begin{proposition}[Structure and defect with a lagged Hessian]
\label{prop:lagged-hessian}
At the current anchor $\bar{\mathbf v}=(\bar{\mathbf q},\bar{\mathbf P})$, replace $\nabla^2H(\bar{\mathbf v})$ by a symmetric matrix $\widetilde{\mathbf H}$ while retaining $\nabla H(\bar{\mathbf v})$. Form the affine map from these blocks and assume $\widetilde{\mathbf B}=\Id+h\widetilde{\mathbf H}_{\mathbf q\mathbf P}$ is invertible. The resulting fixed-anchor affine map is symplectic. If $\nabla^2H$ is $M_H$-Lipschitz on the segment between $\bar{\mathbf v}$ and the realized $\mathbf v=(\mathbf q,\mathbf P)$, its implicit-equation defect obeys
\begin{equation}
 \norm{\mathbf d}
 \leq h\norm{\nabla^2H(\bar{\mathbf v})-\widetilde{\mathbf H}}\,\norm{\mathbf v-\bar{\mathbf v}}
       +\tfrac12 hM_H\norm{\mathbf v-\bar{\mathbf v}}^2+\norm{\mathbf e_{\rm lin}},
 \label{eq:stale-defect}
\end{equation}
where $\mathbf e_{\rm lin}$ is the equation residual of the affine linear solve. In exact arithmetic, a consistent anchor has zero defect.
\end{proposition}
\begin{proof}
Symmetry of the diagonal Hessian blocks and transposed cross blocks gives the factorization in Eq.~\eqref{eq:symplectic-factorization}, with the replaced blocks. Translation does not affect that identity. For the defect, subtract the lagged linearized equations from the exact implicit equations, add and subtract
$\nabla^2H(\bar{\mathbf v})(\mathbf v-\bar{\mathbf v})$, and bound the Taylor remainder as in Eq.~\eqref{eq:mixed-defect}.
\end{proof}
If $\widetilde{\mathbf H}=\nabla^2H(\mathbf v_{\rm old})$, the first matrix norm is at most $M_H\norm{\bar{\mathbf v}-\mathbf v_{\rm old}}$. This first-order term quantifies the effect of lagged curvature on the local refinement budget.

\paragraph{Paired HamiBalls-1 validation.}
We compare full refresh with the period-16 PLAS-S policy in Appendix~\ref{app:integrator-ablation} on matched 48-edge inputs: 512 sources, two noise realizations, $N=20$, and unchanged weights and other settings. Table~\ref{tab:hessian-refresh} uses FP64 subtraction and accumulation to compute prediction errors.

\begin{table}[htbp]
\centering\small
\caption{Prediction accuracy with full and low-frequency Hessian refresh on HamiBalls-1 over 48 edges; paired timing is described in Appendix~\ref{app:inference-wallclock}. Win rates use noise-mean pointwise errors against PhysiFormer.}
\label{tab:hessian-refresh}
\begin{tabular}{@{}lrrrr@{}}
\toprule
Implementation & $z$ MSE & $q$ MSE & $p$ MSE & Win (\%)\\
\midrule
Ours & 0.04584945 & 0.01450559 & 0.07719332 & 64.4263\\
Ours (PLAS-S) & 0.04584186 & 0.01450082 & 0.07718289 & 64.4173\\
\bottomrule
\end{tabular}
\end{table}

PLAS-S reduces median sampling latency by about 6\% in paired timing (Appendix~\ref{app:inference-wallclock}), with essentially unchanged aggregate prediction accuracy in this 48-edge HamiBalls-1 comparison. Relative to full refresh, total MSE changes by $-0.01657\%$ (source-cluster bootstrap 95\% interval $[-0.04862\%,+0.01055\%]$) and win rate by $-0.00895$ percentage points. One of 512 sources exceeds a 1\% MSE increase, with a worst increase of 2.55\%.

\subsection{Inference wall-clock}
\label{app:inference-wallclock}
We compare prediction accuracy at similar measured inference cost on HamiBalls-1 using PhysiFormer with $N=20$ and $N=40$ RF steps and Ours with $N=20$. PhysiFormer at $N=40$ retains the original model weights and denominator floor of $0.05$. Prediction MSE is evaluated on 512 sources and two noise realizations over 192 physical edges without resets.

We time complete 48-edge sampling calls on one NVIDIA A800 80GB PCIe GPU with PyTorch 2.5.1, CUDA 12.4, FP32, and batch size 64. CUDA is synchronized before and after each call. Measurements use warmed execution, excluding model loading, data I/O, compilation, and graph capture.

\begin{table}[htbp]
\centering\small
\caption{Inference cost and prediction accuracy on HamiBalls-1, evaluated on 512 sources and two noise realizations. Latency is the median seconds per batch of 64 48-edge windows; normalized phase-space MSE covers all 192 predicted edges.}
\label{tab:inference-wallclock}
\begin{tabular}{@{}lrrr@{}}
\toprule
Method & RF steps $N$ & Seconds / batch & 192-edge MSE\\
\midrule
PhysiFormer & 20 & 0.402 & 0.34029\\
PhysiFormer & 40 & 0.809 & 0.32767\\
Ours (PLAS) & 20 & 0.786 & 0.25071\\
\bottomrule
\end{tabular}
\end{table}

At comparable measured latency, Ours achieves 23.5\% lower MSE than PhysiFormer with $N=40$ (Table~\ref{tab:inference-wallclock}). Reducing PhysiFormer's denominator floor to $0.025=1/N$, so the final Euler step returns the clean prediction, gives a 192-edge MSE of $0.343529$ under the same evaluation protocol. The accuracy advantage therefore persists when PhysiFormer receives a slightly larger wall-clock budget.

The sampler employs compiled derivative and mixed-scan kernels, batched Hamiltonian derivatives, and cached invariant computations. Ours in the main text uses full Hessian refresh; PLAS-S is an optional variant (Appendix~\ref{app:lagged-hessian}).

\section{Towards Scaling Up}
\label{app:scaling}
\subsection{Diagonal--low-rank affine propagation}
PLAS-DLR uses diagonal--low-rank curvature to construct affine symplectic maps. Let $m=dn$ be the number of position coordinates and $2m$ the phase-space dimension. At anchor $\mathbf a$, it retains the exact gradient and approximates the Hessian by
\begin{equation}
 \widetilde{\mathbf S}=\diag(\mathbf d)+\mathbf R\mathbf C\mathbf R^\top,
 \qquad \mathbf R\in\R^{2m\times r},\quad \mathbf C=\mathbf C^\top\in\R^{r\times r}.
\end{equation}
The diagonal is estimated with $s$ Rademacher Hessian--vector products (HVPs), averaging $\mathbf v\odot(\nabla^2H(\mathbf a)\mathbf v)$. A randomized range approximation of the remaining curvature supplies $\mathbf R$; projection and symmetrization give $\mathbf C$. Probe directions are held fixed during inference. We use $r=2$ and $s=4$.

Write $\mathbf R=[\mathbf R_q^\top,\mathbf R_p^\top]^\top$. The implicit block becomes $\widetilde{\mathbf B}=\Id_m+h\mathbf R_q\mathbf C\mathbf R_p^\top$. The Woodbury identity gives
\begin{equation}
 \widetilde{\mathbf B}^{-1}\mathbf v=\mathbf v-h\mathbf R_q\mathbf C
 (\Id_r+h\mathbf R_p^\top\mathbf R_q\mathbf C)^{-1}\mathbf R_p^\top\mathbf v.
\end{equation}
Thus propagation uses an $r\times r$ solve and tall-factor products. The symmetric quadratic Hamiltonian defined by $\widetilde{\mathbf S}$ has symmetric pure Hessian blocks, so Theorem~\ref{thm:affine}'s factorization preserves the symplecticity of the resulting affine maps whenever $\widetilde{\mathbf B}$ is invertible.

For a Hessian with Lipschitz constant $L_H$ and curvature error $\norm{\widetilde{\mathbf S}-\nabla^2H(\mathbf a)}\leq\epsilon$, Taylor's theorem yields
\begin{equation}
 \norm{\nabla H(\mathbf z)-\nabla H(\mathbf a)-\widetilde{\mathbf S}(\mathbf z-\mathbf a)}
 \leq\epsilon\norm{\mathbf z-\mathbf a}+\tfrac12 L_H\norm{\mathbf z-\mathbf a}^2.
\end{equation}
The generating-relation defect is multiplied by $h$; a uniformly stable implicit solve transfers this estimate to state error. Bounded curvature error and $O(h)$ anchor mismatch therefore give an $O(h^2)$ local contribution.

\subsection{Computational and memory complexity}
Let $C_{\mathrm{HVP}}(m)$ denote the cost of one HVP. Dense Hessian construction requires $O(mC_{\mathrm{HVP}}(m))$ work, followed by $O(m^3)$ dense factorization and $O(m^2)$ affine application. PLAS-DLR construction costs
\begin{equation}
 O\big((s+2r)C_{\mathrm{HVP}}(m)+mr^2+r^3\big),
 \qquad\text{with application cost }O(mr+r^2).
\end{equation}
For batch size $b$ and window length $L$, stored jets require $O(bLm^2)$ space for PLAS and $O(bL(mr+r^2))$ for PLAS-DLR. The packed DLR representation stores $3mr+6m+r^2$ scalars per anchor, compared with $4m^2+2m$ for a dense affine map. At $d=2,n=100,r=2$, these are 2,404 and 160,400 scalars. With fixed width, rank, and probe count, full-attention HVPs have quadratic object-count work, giving a subcubic construction route. Total inference memory also includes attention, automatic-differentiation intermediates, and model buffers; its measured scaling is shown below.

\begin{figure}[htbp]
\centering
\includegraphics[width=\textwidth]{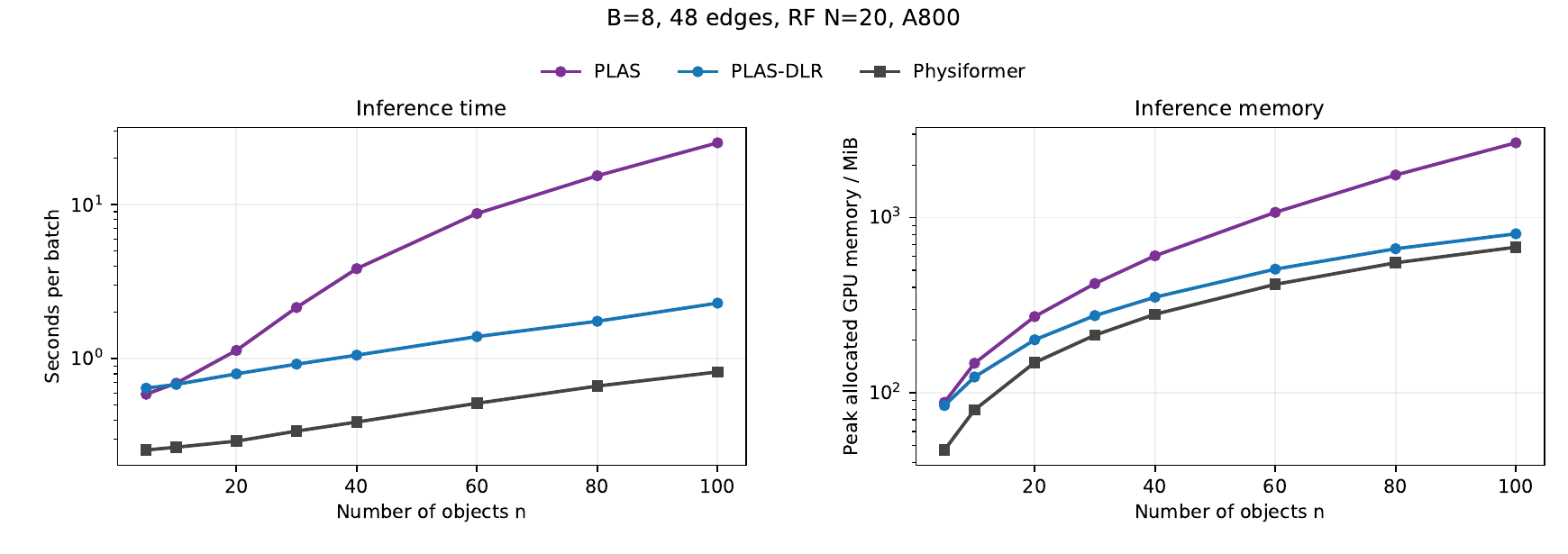}
\caption{Object-count scaling on an A800. Left: median inference time after three warm-up calls, using five timed calls. Right: peak allocated GPU memory after warm-up, measured in separate processes for each method and size. Both vertical axes are logarithmic. Batch size is 8, window length is 48, and RF uses 20 steps.}
\label{fig:scaling}
\end{figure}
We extend the two-dimensional token interface to $n\in\{5,10,20,30,40,60,80,100\}$ using replicated physical attributes and Gaussian noisy trajectories. Figure~\ref{fig:scaling} measures computational scaling with the same tensor modules. At $n=100$, PLAS-DLR takes 2.297 seconds per batch, compared with 0.821 seconds for PhysiFormer. Their inference allocation peaks are 808.73 and 678.39 MiB, respectively, compared with 2,678.30 MiB for PLAS. Allocated memory includes model tensors, inputs, outputs, and live computation-graph buffers; allocator reservations and CUDA context are separate. Hessian construction is batched to fit the working-memory budget.

\subsection{HamiBalls-1 accuracy}
We evaluate 64 physical sources with two noise realizations over 48 edges. PLAS-DLR obtains normalized $q$-MSE 0.0144733, $p$-MSE 0.0772124, and $z$-MSE 0.0458429. The corresponding PLAS values are 0.0144720, 0.0772396, and 0.0458558. Together with the runtime and memory curves, these results support diagonal--low-rank curvature as a practical route to larger systems while retaining trajectory accuracy on HamiBalls-1.

\section{Experimental Setup and Additional Results}
\label{app:results}
\subsection{Dataset generation}
\label{app:dataset-generation}
Both datasets contain independently simulated mechanical episodes with fixed object identities and observable physical attributes. Each episode spans 192 edges (193 saved states), with a saved-frame interval of $h=1/30$ and total duration 6.4. Table~\ref{tab:dataset-generation} summarizes the physical settings and the train/validation portions used in this work. HamiBalls-2 balances object counts across $n=5,\ldots,10$, with equal training counts at each $n$.

\begin{table}[htbp]
\centering\small
\caption{Dataset generation settings. $\mathcal U[a,b]$ denotes uniform sampling; speed ranges apply before the HamiBalls-1 center-of-mass adjustment. Restitution entries are per-body material coefficients.}
\label{tab:dataset-generation}
\setlength{\tabcolsep}{5pt}
\renewcommand{\arraystretch}{1.1}
\begin{tabularx}{\linewidth}{@{}lXX@{}}
\toprule
Setting & HamiBalls-1 & HamiBalls-2\\
\midrule
Domain & $[-1,1]^2$ & $[-1,1]^2\times[0,2]$\\
Objects & 5 disks & 5--10 spheres\\
Mass; radius & $\mathcal U[0.5,1.5]$; $\mathcal U[0.06,0.10]$ & Same ranges\\
Initial speed & $\mathcal U[0.35,0.90]$ & $\mathcal U[0.35,0.90]$\\
Body restitution & $\mathcal U[0.4,0.9]$ & $\mathcal U[0.90,0.98]$\\
Boundary restitution & $0.8$ & $1.0$\\
Gravity magnitude & $0$ & $1$ (downward)\\
Springs & Central harmonic, $k=0.5$ & Sparse pairwise Hookean\\
Simulator & Pymunk 7.3.0 / Chipmunk2D & PyBullet 3.2.7\\
Substeps per saved edge & 8 & 32\\
Contact solver iterations & 20 & 50\\
Train / validation episodes & 14,336 / 512 & 30,720 / 512\\
\bottomrule
\end{tabularx}
\end{table}

\paragraph{Initial states and spring connections.}
Mass, radius, and restitution are sampled independently. Centers are sampled uniformly within the admissible box, rejecting initial overlaps with a clearance of 0.015. Initial velocity directions are isotropic; HamiBalls-1 additionally subtracts the mass-weighted center-of-mass velocity. Its central force is $\mathbf F_i=-0.5\mathbf q_i$. In HamiBalls-2, each unordered pair is connected with probability $2/(n-1)$, conditioned on at least one spring; connectivity is otherwise unrestricted. For each spring, sample a period $T_{ij}$ log-uniformly from $[\sqrt{0.5},\sqrt{8}]$ and set $k_{ij}=\mu_{ij}(2\pi/T_{ij})^2$, where $\mu_{ij}=m_im_j/(m_i+m_j)$. Its rest length is the initial separation multiplied by $\exp(u)$, with $u\sim\mathcal U[-0.15,0.15]$.

\paragraph{Simulation and state representation.}
HamiBalls-1 applies a symplectic-Euler force kick before each Pymunk drift/contact step. HamiBalls-2 recomputes spring forces at every PyBullet substep and advances under gravity and contact impulses. Both use frictionless contacts and no velocity damping; the engines combine colliding bodies' restitution coefficients multiplicatively. PyBullet uses a restitution velocity threshold of 0.2. Episodes are simulated continuously across the full horizon. Saved states contain position and canonical momentum $\mathbf p_i=m_i\mathbf v_i$, together with mass, radius, restitution, and the HamiBalls-2 spring graph, stiffnesses, and rest lengths. Variable-size scenes use masked padding. Fixed normalization scales are computed from the training split. Training uses 48-edge windows, while the reported 192-edge evaluation begins at episode initialization, as detailed below.

\subsection{Aggregation and evidence scope}
Under the protocol in Section~\ref{sec:experiments}, all component MSEs use fixed normalized phase-space coordinates. For an object--edge stratum $\mathcal S$, component dimension $d_c$ ($d_z=4$, $d_p=d_q=2$ for HamiBalls-1; $d_z=6$, $d_p=d_q=3$ for HamiBalls-2), and two noise realizations, the absolute metric is
\begin{equation}
 \MSE_{c,\mathcal S}=\frac{1}{2|\mathcal S|d_c}
 \sum_{(b,k,i)\in\mathcal S}\sum_{\omega=1}^{2}
 \norm{\widehat{\mathbf x}^{(\omega)}_{bki,c}-\mathbf x^*_{bki,c}}_2^2.
 \label{eq:metric}
\end{equation}
Tables~\ref{tab:main} and~\ref{tab:long} use this reduction over their respective intervals. HG-DPF uses matched evaluation with SNIS guidance; deterministic Transformer-AR repeats each source's prediction across both noise labels for aggregation.

For fixed trained models, 95\% percentile CIs for the 192-edge phase-space MSE reductions are $[23.8\%,28.8\%]$ over PhysiFormer on HamiBalls-1 and $[20.7\%,22.1\%]$ over DiT on HamiBalls-2. We use 10,000 paired bootstrap resamples of 512 sources, retaining both noises and all objects and edges per source and recomputing pooled MSE from summed squared errors and valid coordinate counts.

Contact marks object--edge cells with positive normal contact impulse in any substep, including sustained contact; all others are continuous, including objects away from simultaneous contacts. Both datasets use force threshold $10^{-12}$, dividing impulse by substep duration on HamiBalls-1. Its $512\cdot192\cdot5=491{,}520$ cells comprise 7,829 contact and 483,691 continuous cells. Contact-mask reconstruction from original scene seeds and accepted sampling attempts reproduced all stored ground-truth trajectories bitwise. Both noises are aggregated before cell/source/track win statistics.

\subsection{Autoregressive baselines on HamiBalls-2}
Deterministic Transformer-AR inputs are normalized positions/momenta, observable attributes, spring graph/parameters, and physical time increment. Object tokens are linearly embedded into four pre-layer-normalized Transformer blocks \citep{vaswani2017attention}: width 340, four heads, width-1360 GELU feed-forward layers, zero dropout. Graph summaries augment object features; spring features provide learned pairwise attention bias. Final layer normalization and a shared linear head predict six-component state increments; padding masks exclude absent objects.

Context-one uses the current state; context-four concatenates the latest four chronologically, zero-padding missing states with history-validity indicators (5,574,318/5,581,798 parameters). Training uses all 48 teacher-forced transitions per window, including shorter initial histories, averaging normalized squared errors over valid objects, edges, and coordinates.

AdamW \citep{loshchilov2019adamw} runs for 50,000 updates with 64-window batches, 1,000-update warmup to $3\times10^{-4}$, cosine decay to $3\times10^{-7}$, weight decay $10^{-4}$, and norm-1 gradient clipping. BF16 autocast retains float32 parameters/loss; evaluation uses final weights.

Inference rolls out 192 edges from the observed initial frame, retaining up to four states for context-four.

Using available ground-truth predecessors at every edge over all 512 sources and 192 edges gives single-step MSEs of 0.004242/0.003419 for context-four/one. Supplying three ground-truth transitions before free rollout leaves a 192-edge MSE gap (1.478274/0.883231), so cold start alone does not explain the difference.

\subsection{Per-seed errors and continuation chunks}
\begin{table}[htbp]
\centering\small
\caption{HamiBalls-1 absolute 48-edge normalized MSE by physical data-generation seed.}
\label{tab:per-seed}
\begin{tabularx}{\textwidth}{l*{6}{>{\centering\arraybackslash}X}}
\toprule
 & \multicolumn{3}{c}{PhysiFormer} & \multicolumn{3}{c}{Ours}\\
\cmidrule(lr){2-4}\cmidrule(l){5-7}
Physical seed & $z$ & $q$ & $p$ & $z$ & $q$ & $p$\\
\midrule
40 & 0.06718 & 0.02263 & 0.11173 & 0.04557 & 0.01520 & 0.07594\\
41 & 0.06582 & 0.02046 & 0.11118 & 0.05146 & 0.01686 & 0.08606\\
42 & 0.06718 & 0.02420 & 0.11016 & 0.04431 & 0.01550 & 0.07311\\
43 & 0.05907 & 0.01574 & 0.10240 & 0.04206 & 0.01045 & 0.07366\\
\bottomrule\end{tabularx}\end{table}

\begin{table}[htbp]
\centering\small
\caption{HamiBalls-1 absolute normalized MSE for disjoint 48-edge continuation chunks.}
\label{tab:chunks}
\begin{tabularx}{\textwidth}{l*{6}{>{\centering\arraybackslash}X}}
\toprule
 & \multicolumn{3}{c}{PhysiFormer} & \multicolumn{3}{c}{Ours}\\
\cmidrule(lr){2-4}\cmidrule(l){5-7}
Physical edges & $z$ & $q$ & $p$ & $z$ & $q$ & $p$\\
\midrule
1--48 & 0.06481 & 0.02076 & 0.10887 & 0.04585 & 0.01451 & 0.07719\\
49--96 & 0.26704 & 0.23468 & 0.29940 & 0.18787 & 0.16078 & 0.21496\\
97--144 & 0.46841 & 0.49981 & 0.43702 & 0.33972 & 0.36346 & 0.31598\\
145--192 & 0.56089 & 0.63230 & 0.48947 & 0.42938 & 0.48507 & 0.37370\\
\bottomrule\end{tabularx}\end{table}

Table~\ref{tab:chunks} separates the four windows; Table~\ref{tab:long} pools the last two. Equation~\eqref{eq:metric} uses each interval's valid coordinate and contact/continuous cell counts.

\subsection{Comparison with the internal diffusion expert}
\label{app:internal-d-control}
\begin{sloppypar}
In Section~\ref{sec:expert-feedback}'s 512-source, paired-noise comparison, internal D-only/Ours MSEs are 0.066289/0.045850 (48 edges) and 0.338632/0.250707 (192 edges): reductions of 30.83\%/25.96\%. With noises averaged per source, Ours improves 82.62\%/78.91\% of trajectories, respectively.
\end{sloppypar}

\subsection{Model-tree and objective ablation: controls and diagnostics}
\label{app:tree-objective}
Table~\ref{tab:tree-objective}'s five variants (D: Ours) share frozen base models, sources, normalization, RF evaluation ($N=20$), batch size 64, and 800 downstream neural updates. Each uses its own training carriers; tree construction, ridge fitting, and calibration add computation.

A/B supervise residuals with mean normalized $q/p$ squared errors on teacher-forced local and mixed autoregressive states, and routers with mixed-prediction squared errors on unrolled histories. The full objective adds robust/recovery and hull residual terms and preference, regret-reweighting, and no-harm router terms. Ridge regularization, features, cumulative fitting, and training-side calibration follow the full configuration. B/D share tree-construction rules, with five/seven reachable leaves fixed during final training; A/C retain shared final $q/p$ heads. GL uses global residual/router readouts with the same features and calibration. Combined neural-parameter/ridge-coefficient counts are 1,176,220 (A/C), 1,178,428 (B/D), and 1,176,496 (GL), with 1,056 (B/D) and 132 (GL) ridge coefficients.

\paragraph{Difficulty stratification.}
Quartiles rank source--edge--object cells by two-noise-mean phase-space error of Ours' uncorrected H candidate on its mixed history. From easiest to hardest, Ours reduces total prediction error relative to PhysiFormer by 93.13\%, 96.21\%, 87.16\%, and 20.66\%. Momentum correction reduces direct candidate SSE by 11.61\% in the hardest quartile and increases it in easier quartiles with small uncorrected errors.

\paragraph{Training-side regional readouts.}
Ours' seven reachable leaves occupy 3.87\%, 1.86\%, 2.20\%, 2.42\%, 8.75\%, 8.57\%, and 72.34\% of readout-fitting rows. Regional momentum readouts reduce direct candidate SSE by 12.95\%, versus GL's 6.10\%, complementing the evaluation in Table~\ref{tab:tree-objective}. Tree inputs exclude event labels and source identifiers.

\addtocontents{toc}{\protect\setcounter{tocdepth}{3}}

\end{document}